\documentclass{article}

\usepackage[margin=1in]{geometry}
\usepackage{times,natbib}

\usepackage{microtype}
\usepackage{graphicx}
\usepackage{booktabs} % for professional tables
\usepackage{multirow}
\usepackage{amsmath}
\usepackage{amssymb}
\usepackage{mathtools}
\usepackage{amsthm}

\usepackage{algorithm}
\usepackage{algorithmic}
\usepackage{float}

\usepackage{hyperref}

\usepackage[capitalize,noabbrev]{cleveref}
\usepackage{diagbox}
\theoremstyle{plain}
\newtheorem{theorem}{Theorem}[section]

\newtheorem{lemma}[theorem]{Lemma}
\newtheorem{corollary}[theorem]{Corollary}
\theoremstyle{definition}
\newtheorem{definition}[theorem]{Definition}

\theoremstyle{remark}
\newtheorem{remark}[theorem]{Remark}
\usepackage{amsmath,amsfonts,bm,physics,bbm}

\usepackage{mathrsfs}
\usepackage{subcaption}   % modern package for subfigures

\newcommand{\Unif}{{\mathrm{Unif}}}

\def\mI{\mathbb{I}}

\renewcommand{\eqref}[1]{\hyperref[#1]{(\ref{#1})}}

\crefname{ALG@line}{line}{lines}
\Crefname{ALG@line}{Line}{Lines}



\newtheorem*{lemma*}{Lemma}

\def\eqref#1{equation~\ref{#1}}
\def\1{\bm{1}}

\newcommand{\npt}{n_{\mathrm{pt}}}
\newcommand{\mpt}{m_{\mathrm{pt}}}
\newcommand{\Sm}{\mcS_{m}}
\newcommand{\Shat}{\widehat{S}}

\def\mI{{\bm{I}}}

\DeclareMathAlphabet{\mathsfit}{\encodingdefault}{\sfdefault}{m}{sl}
\SetMathAlphabet{\mathsfit}{bold}{\encodingdefault}{\sfdefault}{bx}{n}

\def\mcA{{\mathcal{A}}}
\def\mcB{{\mathcal{B}}}

\def\mcD{{\mathcal{D}}}

\def\mcM{{\mathcal{M}}}
\def\mcN{{\mathcal{N}}}

\def\mcS{{\mathcal{S}}}

\def\mcU{{\mathcal{U}}}
\def\mcV{{\mathcal{V}}}

\def\sI{{\mathbb{I}}}

\newcommand{\E}{\mathbb{E}}

\DeclareMathOperator*{\argmin}{arg\,min}

\crefname{algline}{line}{lines}  % Singular and plural forms
\Crefname{algline}{Line}{Lines}  % Capitalized versions

\newcommand{\algfootnote}[1]{%
    \leavevmode\unskip\strut\vadjust{\hbox to 0pt{\hss\footnotemark}}%
    \addtocounter{footnote}{-1}\footnotetext{#1}%
}
\usepackage[textsize=tiny]{todonotes}

\title{Auditing of Unlearning Algorithms}

\author{%
  \begin{tabular}{cc}
    Sahasrajit Sarmasarkar & Anastasia Koloskova \\[0.2em]
    Stanford University & University of Zurich \\[0.2em]
    \texttt{sahasras@stanford.edu} & \texttt{anastasiia.koloskova@uzh.ch} \\[1em]
    \multicolumn{2}{c}{Sanmi Koyejo} \\[0.2em]
    \multicolumn{2}{c}{Stanford University} \\[0.2em]
    \multicolumn{2}{c}{\texttt{sanmi@cs.stanford.edu}}
  \end{tabular}
}
\date{\today}

\begin{document}

\maketitle

\begin{abstract}
    Evaluating whether unlearning algorithms truly remove training data influence remains an open challenge. We propose a practical auditor that computes data-dependent lower bounds on the unlearning parameter $\varepsilon$ using membership inference attacks. Evaluating multiple unlearning algorithms, we find a sharp separation: algorithms with rigorous guarantees, such as model clipping and rewind-to-delete, achieve very small $\varepsilon$ bounds that do not falsify their unlearning guarantees, whereas empirical methods such as Hessian-based unlearning, interleaved ascent–descent, ascent on the forget set, and fine-tuning on the retain set exhibit large bounds, indicating poor unlearning. Our auditor provides a practical tool for empirically falsifying unlearning claims through a hypothesis-testing framework, and we validate it on CIFAR-100 and Shakespeare text. \footnote{Code available at: \url{https://github.com/Sahasrajit123/audit-unlearning-code}}.
\end{abstract}

\section{Introduction}

Machine unlearning is the task of removing the influence of specific training samples from an already-trained model, ideally yielding a model that behaves as if those samples had never been seen during training
\cite{cao2015towards}. Simply deleting a data point from the training set does not achieve this: its information remains embedded in the learned parameters and can often be recovered by an adversary with query
access, e.g.\ via membership inference attacks \cite{DBLP:journals/corr/ShokriSS16, carlini2022membershipinferenceattacksprinciples}. This motivates \emph{exact} unlearning, which produces a model with the same distribution as one retrained from scratch on the retained data \cite{bourtoule2021machine,ginart2019datadeletion}, and
\emph{approximate} unlearning, which relaxes this requirement in
exchange for efficiency \cite{guo2020certified,sekharimachineunlearningconvex,neel2021descent,golatkar2020eternal,kurmanji2023unbounded}. %The problem is particularly relevant in light of privacy legislation such as the EU's GDPR and its \emph{right to be forgotten}, under which controllers must erase personal data ``without undue delay'' \cite{gdpr2016}.

While exact unlearning offers the strongest guarantees, it is largely impractical at scale: existing schemes are either restricted to simple models such as $k$-means \cite{ginart2019datadeletion} or rely on data sharding \cite{bourtoule2021machine}, trading utility for cheap deletion. To scale to deep networks, a parallel line of work develops \emph{approximate} unlearning algorithms, with two kinds of compromise. The first family is \emph{heuristic} and provides no formal guarantee, including Fisher/Hessian-based parameter scrubbing \cite{golatkar2020eternal}, gradient ascent on the forget set, fine-tuning on the retain set, and interleaved ascent--descent schemes such as \textsc{SCRUB} \cite{kurmanji2023unbounded}. The second family adopts the \emph{certified} $(\varepsilon, \delta)$-indistinguishability notion of \citet{guo2020certified}, borrowed from differential privacy, which requires the unlearned model to be $(\varepsilon, \delta)$-indistinguishable \cite{dworkdpbook} from one retrained from scratch. This guarantee bounds the distinguishing advantage of any test operating on the unlearnt model relative to the baseline retrained model. Most certified-unlearning algorithms require (strong) convexity of the loss \cite{guo2020certified,sekharimachineunlearningconvex,neel2021descent,qiao2025hessianfree,zhang2025certifiedunlearningdeepneural} or a unique minimiser \cite{Allouah2025UtilityComplexityUnlearning}; only a handful of recent methods provide certified guarantees for genuinely nonconvex losses \cite{koloskova2025certifiedunlearningneuralnetworks,rewind2delete,chienlangevianunlearning,chaurasiadatadeletionunlearning}.

While the bounds above are theoretical, our goal in this paper is to design an auditor that uses empirical evidence to test whether a claimed unlearning guarantee actually holds. Formally, our goal is to reject the hypothesis $\varepsilon < \varepsilon_{\text{LB}}$ for any certified $(\varepsilon,\delta)$ unlearning algorithm in the spirit of recent work on differential-privacy auditing \cite{jagielski2020auditing,nasr2023tight,steinke2023privacyauditing1training}.

While our audit draws on differential-privacy auditing, the unlearning setting is closer in spirit to \emph{group privacy}~\cite{dworkdpbook}: certified unlearning gives $(\varepsilon,\delta)$-indistinguishability under deletion of an \emph{arbitrary subset} of forget points rather than a single point, and the two notions differ by a multiplicative blow-up in $\varepsilon$. Off-the-shelf DP auditors therefore certify only the weaker single-point neighbouring guarantee and yield loose $\varepsilon$ bounds in the unlearning setting; our meta-algorithm (\cref{metaalg:auditor_unlearning}) extends the one-run auditor of \citet{steinke2023privacyauditing1training} with new bounds tailored to subset-level guarantees. Our key contributions are summarised below.

\begin{itemize}
	\item \textbf{Threat model.} We adopt a strong adversary with black-box knowledge of the learning algorithm, the unlearning algorithm, the data-loading shuffler, and the initial model weights at the start of training. The adversary knows the underlying random mechanism used by these procedures but does not observe its specific random realisation on any given run.
	
	 \item \textbf{Auditor.} Leveraging the definition of $(\varepsilon,\delta)$-unlearning with respect to arbitrary forget subsets, we design a hypothesis test that rejects the null hypothesis $\varepsilon < \varepsilon_{\mathrm{LB}}$ with controlled Type~I error for the case of $\delta=0$ (see \cref{lemma:lower_bound_overlap_score}).

	\item \textbf{Empirical evaluation.} We develop two instantiations of this auditor in \cref{sec:auditor_instantiations} and evaluate them across a range of unlearning algorithms, revealing a sharp separation between certified and uncertified methods: certified unlearning algorithms yield very small $\varepsilon$ lower bounds, whereas uncertified algorithms produce lower bounds that frequently exceed $50$--$60$.
\end{itemize}

\section{Related work}

Empirically lower-bounding the privacy parameter $\varepsilon$ of an $(\varepsilon,\delta)$-DP mechanism has been studied in \cite{jagielski2020auditing,nasr2021adversaryinstantiationlowerbounds,nasr2023tight,selvatightauditingdp}, typically by planting canaries in a pair of neighbouring datasets and converting distinguishing accuracy into an $\varepsilon$ lower bound via the hypothesis-testing interpretation of DP \cite{kairouz2015composition,Wasserman01032010}. These approaches require many independent training runs to drive the audit's Type~I/II error down. More recent work \cite{steinke2023privacyauditing1training,Saeedfdpauditing} reduces the run cost to a \emph{single} training run for $(\varepsilon,\delta)$-DP and $f$-DP respectively, by planting many independent canaries and auditing them jointly within one run.

Our auditor borrows the joint-canary idea from \cite{steinke2023privacyauditing1training}, but the transfer to unlearning is not direct. As argued above, the subset-level guarantee yields weaker per-canary signal than the single-point neighbouring relation underlying standard DP audits, so we aggregate across multiple runs to obtain tight $\varepsilon$ lower bounds, and converting the joint canary statistics into a valid bound under this neighbouring relation requires new lemmas. To tighten the resulting bounds further, we replace the loss-difference scoring of \cite{steinke2023privacyauditing1training,Saeedfdpauditing} with a LiRA-style \cite{carlini2022membershipinferenceattacksprinciples} per-canary likelihood-ratio test on logit-scaled confidences, fitting Gaussians to the canary's score under the in-forget-set and never-seen hypotheses.

A separate line of work audits unlearning directly rather than DP: \emph{backdoor-based verification} (e.g.\ Athena \cite{sommer2022athena}, though \citet{zhang2024fragile} show such schemes can be gamed by a dishonest provider), \emph{output-difference} audits such as TAPE \cite{wang2025tape} that train a reconstructor on the gap between pre- and post-unlearning models, and \emph{sample-level} audits returning per-example unlearning-completeness scores \cite{wang2025iam,triantafillou2024progress,gu2025auditingdp}. Our auditor instead returns a quantity calibrated to the certified-unlearning parameter $\varepsilon$ itself, rather than a per-example completeness score: an empirical lower bound on the algorithm's $\varepsilon$ that is directly comparable to its claimed $(\varepsilon, \delta)$ guarantee.

% Our auditor instead returns an $\varepsilon$ lower bound on the certified-unlearning parameter of the \emph{algorithm itself}, directly comparable to its claimed $(\varepsilon,\delta)$ guarantee.

Early MIA formulations \cite{DBLP:journals/corr/ShokriSS16,yeom2018privacy} were sharpened by likelihood-ratio attacks (LiRA \cite{carlini2022membershipinferenceattacksprinciples}, RMIA \cite{zarifzadeh2024rmia}). In the unlearning setting, \citet{chen2021jeopardizes,bertran2024reconstruction} exploit the \emph{difference} between original and unlearned models to infer or reconstruct deleted samples. These formulations often assume access to  both pre and post unlearning models where as our audit is formulated around access to just the unlearnt model.

%%--- a strictly stronger threat model than ours, since certified unlearning compares the unlearned model to an unobserved retrain-from-scratch counterfactual and our auditor only requires query access to the unlearned model.

\section{Problem setup}

\subsection{Unlearning definition}

Let $\mcA$ be a training algorithm that, given a dataset $\mcD$, outputs a trained model $\mcA(\mcD)$. Suppose a subset $\mcD_f \subseteq \mcD$, called the \emph{forget set}, is requested to be removed. We write $\mcD_r := \mcD \setminus \mcD_f$ for the corresponding \emph{retain set}.

A natural baseline for unlearning is to retrain the model from scratch on $\mcD_r$. However, full retraining is often computationally expensive. The goal of an unlearning algorithm is therefore to produce, more efficiently, a model whose distribution is close to that of a suitable retraining procedure on the retain set. Formally, an \emph{unlearning algorithm} $\mcU$ takes as input the trained model, the full dataset, and the forget set, and outputs an \emph{unlearned} model $f_u = \mcU\big(\mcA(\mcD), \mcD, \mcD_f\big)$.

%%\ak{tell that the idea is to find model close to retraining from scratch, but more efficiently. There are notions of exact and approx unlearning (maybe actually this last part should be mentioned before - in the intro or related work sections).}

\subsection{$(\varepsilon,\delta)$ indistinguishability \cite{dworkdpbook}}

%We use the standard notion of $(\varepsilon,\delta)$-indistinguishability \cite{dworkdpbook}.

\begin{definition}[$(\varepsilon,\delta)$-indistinguishability]
Let $X$ and $Y$ be random variables over a common domain $\Omega$. We say that $X$ and $Y$ are $(\varepsilon,\delta)$-indistinguishable, and write $X \approx_{\varepsilon,\delta} Y$, if for every measurable set $S \subseteq \Omega$,  $\Pr[X \in S] \le e^{\varepsilon}\Pr[Y \in S] + \delta$ and 
$\Pr[Y \in S] \le e^{\varepsilon}\Pr[X \in S] + \delta$.
\end{definition}

\subsection{Certified unlearning}
%We use the standard notion of $(\varepsilon,\delta)$-indistinguishability \cite{dworkdpbook}.
%\begin{definition}[$(\varepsilon,\delta)$-indistinguishability]
%Let $X$ and $Y$ be random variables over a common domain $\Omega$. We say that $X$ and $Y$ are $(\varepsilon,\delta)$-indistinguishable, and write $X \approx_{\varepsilon,\delta} Y$, if for every measurable set $S \subseteq \Omega$,
%$\Pr[X \in S] \le e^{\varepsilon}\Pr[Y \in S] + \delta$ and $\Pr[Y \in S] \le e^{\varepsilon}\Pr[X \in S] + \delta$.
%\end{definition}

\begin{definition}[$(\varepsilon,\delta)$-certified unlearning, \citealp{guo2020certified,koloskova2025certifiedunlearningneuralnetworks}]
\label{def:certified_unlearning_general}
We say that $\mcU$ is an $(\varepsilon,\delta)$-certified unlearning algorithm for $\mcA$ if there exists a reference algorithm $\overline{\mcA}$ such that for every forget set $\mcD_f \subseteq \mcD$ with retain set $\mcD_r := \mcD \setminus \mcD_f$, the random variables $\mcU(\mcA(\mcD), \mcD, \mcD_f)$ and $\overline{\mcA}(\mcD_r)$ are $(\varepsilon,\delta)$-indistinguishable.
\end{definition}
Informally, an observer cannot distinguish the unlearned model from one produced using only the retain set, except up to $(\varepsilon, \delta)$.

\subsection{Threat model}
The general definition above leaves the adversary's auxiliary information unspecified. To make the audit concrete, we fix the following threat model. We consider a black-box adversary that additionally observes (i) the initial model $x_o$ at the start of training, and (ii) the shuffling order $\pi$ of training data-points within each epoch; the adversary has no access to any other sources of randomness used during training or unlearning. These assumptions strengthen the adversary relative to the standard threat model, and crucially do \emph{not} weaken the certified-unlearning guarantees of certified unlearning algorithms: both the model clipping algorithm (a.k.a.\ noisy fine-tuning) of \citet{koloskova2025certifiedunlearningneuralnetworks} and rewind-to-delete (R2D) \cite{rewind2delete} remain $(\varepsilon,\delta)$-certified under this stronger adversary.\footnote{R2D is formally analysed only in the full-batch training setting, and our threat model is consistent with that regime.}

Because the adversary observes $x_o$ and $\pi$, the indistinguishability requirement must hold conditional on them; the reference algorithm is allowed to depend on $(x_o, \pi)$ as well. Let $\mcA_{x_o, \pi}(\mcD)$ denote the learning algorithm started at $x_o$ with shuffling order $\pi$.
\begin{definition}[$(\varepsilon,\delta)$-certified unlearning under our threat model]
\label{def:certified_unlearning_threat}
We say that $\mcU$ is an $(\varepsilon,\delta)$-certified unlearning algorithm for $\mcA$ if for every initial model $x_o$ and shuffling order $\pi$, there exists a reference algorithm $\overline{\mcA}_{x_o, \pi}$ such that for every forget set $\mcD_f \subseteq \mcD$, the random variables $\mcU(\mcA_{x_o, \pi}(\mcD), \mcD, \mcD_f)$ and $\overline{\mcA}_{x_o, \pi}(\mcD_r)$ are $(\varepsilon, \delta)$-indistinguishable.
\end{definition}

Our goal is to design an \emph{auditor} (or adversary) that, given black-box access to $\mcA$ and $\mcU$ together with the initial model $x_o$ and shuffling order $\pi$ from our threat model, produces a data-dependent lower bound $\varepsilon_{\mathrm{LB}}$ on the true unlearning parameter $\varepsilon$. We formalise the audit as a one-sided hypothesis test: for a confidence level $\zeta \in (0, 1)$, the auditor outputs $\varepsilon_{\mathrm{LB}}$ such that, under $H_0 : \varepsilon \le \varepsilon_{\mathrm{LB}}$, the test rejects with probability at most $\zeta$. Whenever the test rejects, the true parameter exceeds $\varepsilon_{\mathrm{LB}}$ with confidence $1 - \zeta$, so a high $\varepsilon_{\mathrm{LB}}$ is direct evidence that the unlearning guarantee, if any, is weak. The main paper focuses on $\delta = 0$; see \cref{sec:redn_local_dp}. A discussion on extension of the test beyond $\delta=0$ is given in \cref{sec:nonzero_delta_lb}.

%%In the appendix, we extend the audit to stronger notions of indistinguishability used by certified unlearning algorithms \cite{koloskova2025certifiedunlearningneuralnetworks,chienlangevianunlearning,rewind2delete}, including Rényi DP, zero-concentrated DP \cite{cryptoeprint:concentratedDP,mironov2017renyi}, and Gaussian DP \cite{dong2022gaussian}, removing the $\delta = 0$ restriction without further assumptions. \footnote{We argue in \cref{????} that a small value of $\delta$ should not affect $\varepsilon_{\text{LB}}$ values a lot.}

\section{The algorithm}\label{sec:algorithm_audit}

We now describe the meta-algorithm underlying our audit; pseudocode is given in \cref{metaalg:auditor_unlearning}. The auditor has black-box query access to $\mcA$ and $\mcU$, which lets it run training-and-unlearning repeatedly on inputs of its choosing. At a high level, each run samples a balanced sign vector $S^{(\ell)}$, uses $S^{(\ell)}$ to decide which candidate batches are placed in the forget set, runs training followed by unlearning, and asks the auditor to predict $S^{(\ell)}$ from the resulting unlearned model. The construction extends the many-canary auditing paradigm used for differential privacy \cite{steinke2023privacyauditing1training,nasr2023tight,pillutla2024unleashing} to the unlearning setting.

\paragraph{Setup.} The training dataset is partitioned into $n$ batches of size $B$. Let $\mcB$ denote the universe of batches; the \emph{candidate forget pool} $\mcD_f \subseteq \mcB$ contains $m$ batches and the \emph{retain set} $\mcD_r$ contains the remaining $n - m$. We further partition $\mcD_f$ into $m$ candidate forget batches $\mcD_f = \bigcup_{j=1}^m \mcD_{f,j}$. Throughout, we fix the initial model $x_o$ and the batch-shuffling order $\pi$ (both observed by the auditor under our threat model), and suppress them from notation. We define the set of approximately balanced sign vectors
\[
\Sm := \{\, s \in \{\pm 1\}^m : 0 \le \#\{j : s_j = -1\} - \#\{j : s_j = +1\} \le 1 \,\},
\]
i.e., vectors with as many $+1$s and $-1$s as possible (one extra $-1$ when $m$ is odd, by convention). For $S \in \mcS_m$, let $\mcD_f(S) := \bigcup_{j : S_j = +1} \mcD_{f,j}$ be the sampled forget set under $S$ and $\mcD(S) := \mcD_r \cup \mcD_f(S)$ the corresponding training set.

\paragraph{Audit procedure.} Fix an even \emph{guess budget} $r$. In each of $L$ independent runs, the auditor:
\begin{enumerate}
    \item samples $S^{(\ell)} \sim \mathrm{Unif}(\mcS_m)$ and computes the unlearned model $f_u^{(\ell)} \gets \mcU\bigl(\mcA(\mcD(S^{(\ell)})),\; \mcD(S^{(\ell)}),\; \mcD_f(S^{(\ell)})\bigr)$
    \item produces a guess $\Shat^{(\ell)} \in \{-1, 0, +1\}^m$ with exactly $r/2$ entries equal to $+1$ and $r/2$ equal to $-1$ (the remaining $m - r$ entries are $0$, indicating abstentions);
    \item records the overlap score $V^{(\ell)} := \sum_{j=1}^{m} \max\{0,\; \Shat_j^{(\ell)} S_j^{(\ell)}\}$, which counts the number of correct nonzero guesses (out of $r$).
\end{enumerate}
The audit reports a summary statistic (mean or median) of $\{V^{(\ell)}\}_{\ell=1}^L$, which \cref{lemma:lower_bound_overlap_score} converts into the lower bound $\varepsilon_{\mathrm{LB}}$.

For analysis, \cref{metaalg:auditor_unlearning} induces a randomized mechanism $\mcM : \Sm \to \{-1, 0, +1\}^m$ with $\mcM(S^{(\ell)}) := \Shat^{(\ell)}$. Although $\mcM$ may invoke $\mcA$ and $\mcU$ arbitrarily many times, the only quantities hidden from it are the sampled sign vector $S^{(\ell)}$ and the internal randomness of $\mcA, \mcU$ within each run (e.g.\ gradient noise) -- the latter being precisely the randomness on which the algorithms' $(\varepsilon, \delta)$-certified guarantees rely. Both are accessed only through the unlearned model $f_u^{(\ell)}$, so $\mcM$'s output is post-processing of $f_u^{(\ell)}$ and the certified-unlearning guarantee passes through unchanged. Under strong unlearning, batches in and out of the forget set are nearly indistinguishable, so any overlap score reliably above the random-guessing baseline $r/2$ is evidence of weak unlearning; \cref{lemma:lower_bound_overlap_score} converts the mean (or median) of $\{V^{(\ell)}\}_{\ell=1}^L$ across independent runs into a precise lower bound on $\varepsilon$, with two instantiations of $\mcM$ given in \cref{sec:auditor_instantiations}.

%\subsection{Reduction to local differential privacy}{\label{sec:redn_local_dp}}

%The reason local differential privacy appears here is simple: if certified unlearning holds, then changing the hidden forget-pattern vector should not substantially change the distribution of the auditor's output. This turns the induced map $\mcM$ into a locally private mechanism over $\Sm$.
%
%Assume $(\mcA,\mcU)$ satisfies $(\varepsilon,0)$-certified unlearning. Then the induced mechanism
%\[
%\mcM : \Sm \to \{-1,0,+1\}^{m}
%\]
%satisfies an all-pairs local privacy guarantee: for every $S_1,S_2 \in \Sm$, the random variables $\mcM(S_1)$ and $\mcM(S_2)$ are $(2\varepsilon,0)$-indistinguishable. Equivalently, $\mcM$ is $(2\varepsilon,0)$-locally differentially private over the domain $\Sm$.
%
%Indeed, let $\mcD_f(S)$ denote the subset of $\mcD_f$ encoded by $S \in \Sm$, and write $\mcD(S) := \mcD_r \cup \mcD_f(S)$. Certified unlearning implies that
%\[
%\mcU\big(\mcA(\mcD(S_1)), \mcD(S_1), \mcD_f(S_1)\big)
%\quad\text{and}\quad
%\overline{\mcA}(\mcD_r)
%\]
%are $(\varepsilon,0)$-indistinguishable, and similarly
%\[
%\mcU\big(\mcA(\mcD(S_2)), \mcD(S_2), \mcD_f(S_2)\big)
%\quad\text{and}\quad
%\overline{\mcA}(\mcD_r)
%\]
%are $(\varepsilon,0)$-indistinguishable. By transitivity, the two unlearned-model distributions are therefore $(2\varepsilon,0)$-indistinguishable. Since $\mcM(S)$ is obtained from the corresponding unlearned model by post-processing, the same guarantee carries over to $\mcM(S_1)$ and $\mcM(S_2)$.

\subsection{Reduction to local differential privacy}\label{sec:redn_local_dp}

We now show that the audit reduces to a standard local-differential-privacy (LDP) lower-bound problem, which lets us import existing tools for LDP auditing. Recall from the previous section that, conditional on the fixed dataset $\mcD$, initial model $x_o$, and shuffling order $\pi$, the auditor induces a randomized mechanism $\mcM : \Sm \to \{-1, 0, +1\}^m$ that maps the sampled sign vector $S$ to a prediction $\mcM(S) := \Shat$. Concretely, for $S \in \Sm$, let $\mcD_f(S) \subseteq \mcD_f$ denote the forget subset encoded by $S$, let
\[
f_u(S) := \mcU\bigl(\mcA(\mcD_r \cup \mcD_f(S)),\, \mcD_r \cup \mcD_f(S),\, \mcD_f(S)\bigr),
\]
and let $f_r := \mcA(\mcD_r)$ denote the retrained reference. The auditor's prediction $\mcM(S)$ is computed from $f_u(S)$ alone (with possibly many queries to $\mcA$ and $\mcU$, but no direct access to $S$ or to the internal randomness of $\mcA, \mcU$).

% \begin{lemma}{\label{lemma:epsilon_lower_bound}}
% 	If \((\mcA,\mcU)\) is \((\varepsilon,0)\)-certified unlearning, then \(\mcM\) is \((2\varepsilon,0)\)-locally differentially private.\footnote{This reduction does not extend cleanly to \(\delta>0\). In particular, if \(X \approx_{\varepsilon,\delta} Y\) and \(Y \approx_{\varepsilon,\delta} Z\), then \(X \approx_{2\varepsilon,(1+e^\varepsilon)\delta} Z\). Even for small \(\delta\), this additive slack can be too large for the resulting LDP reduction to be useful.}
% \end{lemma}

\begin{lemma}\label{lemma:epsilon_lower_bound}
If $(\mcA, \mcU)$ is $(\varepsilon, 0)$-certified unlearning under our threat model (\cref{def:certified_unlearning_threat}), then $\mcM$ is $(2\varepsilon, 0)$-locally differentially private \cite{Bebensee2019LocalDP}; that is, for all $S, S' \in \Sm$ and every $T \subseteq \{-1, 0, +1\}^m$, $\Pr[\mcM(S) \in T] \le e^{2\varepsilon}\, \Pr[\mcM(S') \in T]$.
\end{lemma}

\begin{proof}
Fix any $S_1, S_2 \in \mcS_m$. Certified unlearning of $(\mcA, \mcU)$ implies $f_u(S_1) \approx_{\varepsilon, 0} f_r$ and $f_u(S_2) \approx_{\varepsilon, 0} f_r$. Using the triangle inequality, $f_u(S_1) \approx_{2\varepsilon, 0} f_u(S_2)$. \footnote{This reduction does not extend cleanly to $\delta > 0$. In particular, if $X \approx_{\varepsilon, \delta} Y$ and $Y \approx_{\varepsilon, \delta} Z$, then $X \approx_{2\varepsilon, (1 + e^\varepsilon)\delta} Z$; even for small $\delta$, this additive slack can be too large for the resulting LDP reduction to be useful.} Since $\mcM(S)$ depends on $S$ only through $f_u(S)$, it is a post-processing of $f_u(S)$, and post-processing preserves privacy.  
\end{proof}

% \begin{proof}
% 	For any \(S_1,S_2 \in \mcS_m\), certified unlearning gives $f_u(S_1) \approx_{\varepsilon,0} f_r$ and $f_u(S_2) \approx_{\varepsilon,0} f_r$.
% 	Hence \(f_u(S_1) \approx_{2\varepsilon,0} f_u(S_2)\). Since \(\mcM\) is a post-processing of \(f_u\), we obtain
% 	$\mcM(S_1) \approx_{2\varepsilon,0} \mcM(S_2)$
% 	which is exactly \((2\varepsilon,0)\)-local differential privacy.
% \end{proof}

\subsection{Lower bound on $\varepsilon$ from overlap score }

\newcommand{\lemmalowerbounddeltazero}[1][]{
	Let $\mcM : \Sm \to \{-1,0,+1\}^{m}$ be an $(\varepsilon,0)$-locally
	differentially private mechanism. Let $\mathcal{T}\subseteq\{-1,0,+1\}^{m}$ denote the subset of
	vectors with exactly $r/2$ entries equal to $+1$ and $r/2$ entries equal to
	$-1$. Set $M' = \binom{m}{\lfloor m/2\rfloor}$ and $K\;:=\;|\mathcal{T}|\;=\;\binom{m}{r/2,\,r/2,\,m-r}$.
	% \[
	% M'\;:=\;|\mathcal{S}|
	% \;=\;
	% \begin{cases}
	% 	\binom{m}{m/2} & m\text{ even}\\
	% 	\binom{m}{\lfloor m/2\rfloor} & m\text{ odd}
	% \end{cases},
	% \qquad
	% K\;:=\;|\mathcal{T}|\;=\;\binom{m}{r/2,\,r/2,\,m-r}.
	% \]
	Let $\{S^{(\ell)}\}_{\ell=1}^{L}$ be independent random vectors drawn uniformly
	from $\Sm$. For each $\ell\in[L]$, let $\Shat^{(\ell)}=\mcM(S^{(\ell)})$,
	assume $\Shat^{(\ell)}\in\mathcal{T}$ almost surely, and define $V^{(\ell)}\;:=\;\sum_{j=1}^{m}\max\!\bigl\{0,\,\Shat^{(\ell)}_{j}S^{(\ell)}_{j}\bigr\}
	\;\in\;\{0,1,\dots,r\}$.
	For each $u\in\{0,1,\dots,r\}$, define the pointwise bound
	\[
	\pi_{\varepsilon}(u)\;:=\,\;
	\sum_{\substack{\alpha_1,\alpha_2\in\{0,\dots,r/2\}\\ \alpha_1+\alpha_2=u}}
	\binom{m-r}{\lceil(m-r)/2\rceil-(\alpha_1-\alpha_2)}
	\binom{r/2}{\alpha_1}\binom{r/2}{\alpha_2}\cdot\,
	\frac{e^{\varepsilon}\;}{e^{\varepsilon}+M'-1}.
	\]
	Define the upper tail $P_{\varepsilon}(v):=\sum_{u= v}^{r}\pi_{\varepsilon}(u)$.
	Then for every $v\in\mathbb{R}$,
	\begin{equation}\label{eq:lb-mean#1}
		\Pr\!\left[\frac{1}{L}\sum_{\ell=1}^{L}V^{(\ell)}\ge v\right]
		\;\le\;
		\inf_{\lambda\ge 0}\exp\!\left(
		L\log\!\Big(\sum_{u=0}^{r}e^{\lambda u}\pi_{\varepsilon}(u)\Big)-\lambda L v
		\right)
	\end{equation}
	\begin{equation}\label{eq:lb-median#1}
		\Pr\!\left(\mathrm{Median}\!\left(\{V^{(\ell)}\}_{\ell=1}^{L}\right)\ge v\right)
		\;\le\;
		\binom{L}{\lceil L/2\rceil}\,P_{\varepsilon}(v)^{\lceil L/2\rceil}.
	\end{equation}
}

\begin{lemma}\label{lemma:lower_bound_overlap_score}
	\lemmalowerbounddeltazero
\end{lemma}
%\begin{lemma}{\label{lemma:lower_bound_overlap_score}}
%	\lemmalowerbound
%\end{lemma}

% In words, if the underlying unlearning guarantee were truly small, then even a good auditor could not achieve overlap scores much larger than chance except with small probability. This lets us convert a successful audit into a lower bound on the certified-unlearning parameter.

% Given an observed statistic $v$ (chosen as either the empirical mean or median of $\{V^{(\ell)}\}_{\ell=1}^{L}$), we obtain a lower bound $\varepsilon_{\mathrm{LB}}$ by setting the right-hand side of the corresponding inequality to a target level $\zeta$, fixing $\delta=0$, and solving for $\varepsilon$. The final reported bound is then divided by $2$, accounting for the reduction from certified unlearning to local differential privacy.

% To reject the null hypothesis $H_0: \varepsilon \le \varepsilon_{\mathrm{LB}}$, it suffices to evaluate the bounds at $\varepsilon=\varepsilon_{\mathrm{LB}}$. This follows because both $\pi_\varepsilon(v)$ and its upper tail $P_\varepsilon(v)$ are monotone non-decreasing in $\varepsilon$, so the right-hand sides of both inequalities increase with $\varepsilon$.

Given an observed statistic $v$ (the empirical mean or median of $\{V^{(\ell)}\}_{\ell=1}^L$), we set the right-hand side of the relevant bound from \cref{lemma:lower_bound_overlap_score} to a target level $\zeta$, fix $\delta = 0$, and solve for $\varepsilon$ to obtain an LDP lower bound; dividing by $2$ to undo the reduction of \cref{lemma:epsilon_lower_bound} yields the reported $\varepsilon_{\mathrm{LB}}$. To reject $H_0 : \varepsilon \le \varepsilon_{\mathrm{LB}}$ at level $\zeta$, it suffices to evaluate the bound at $\varepsilon = \varepsilon_{\mathrm{LB}}$, since both bounds are monotone non-decreasing in $\varepsilon$ (\cref{cor:bound_monotonicity}).

\begin{remark}[Auditor design choices]\label{remark:auditor_design_choices}
Two parameters jointly control the audit's power: the number of candidate forget batches $m$ (equivalently batch size $B = |\mcD_f|/m$) and the support size $r$.

\paragraph{Choice of $m$ (or $B$).} Larger $m$ raises the maximum attainable bound, since identifying the correct sign vector among $|\Sm|$ exponentially many candidates is information-theoretically harder. But larger $m$ also means smaller batches and weaker per-batch leakage, so the right operating point depends on the algorithm: heuristic methods (e.g.\ gradient ascent on forget, fine-tuning on retain) leak strongly per batch and prefer small $B$ (large $m$), while certified methods cap per-batch influence through their $(\varepsilon, \delta)$ budget and require larger $B$ (smaller $m$) for the auditor to beat random guessing.

\paragraph{Choice of $r$.} The support size $r$ lets the auditor abstain on uncertain batches. The empirical lower bound traces an inverted-U in $r$: it rises with $r$ as the maximum attainable bound grows and confident predictions are still available, then falls once the auditor is forced to commit to low-confidence batches whose errors dilute the overlap score. We study these trade-offs empirically in \cref{sec:appendix_batch_size_varying_study}.
\end{remark}
%
% \begin{remark}\label{remark:auditor_design_choices}
% A few facts about \cref{lemma:epsilon_lower_bound}:
% \begin{itemize}
% \item \textbf{Dependence on $m$.} The maximum lower bound implied by a perfect prediction grows with $m$, since identifying the correct sign vector among $|\mcS_m|$ exponentially many candidates is much harder at large $m$.
% \item \textbf{Choice of batch size $B$.} Since $m$ scales inversely with $B$, smaller batches enlarge $\mcS_m$ and raise the attainable bound — but only if the adversary can still pick the right hypothesis. Uncertified algorithms admit a strong per-batch signal, so small $B$ (large $m$) is preferable; certified algorithms have weak per-batch signal, so larger $B$ (smaller $m$) is needed to obtain any nonzero bound.
% \end{itemize}
% A detailed empirical study of the $B$ and $r$ tradeoffs, including the inverted-U behaviour in $r$ and concrete maximum-attainable-bound numbers, is given in \cref{sec:appendix_batch_size_varying_study}.
% \end{remark}

\section{Two instantiations of the auditor meta-algorithm}\label{sec:auditor_instantiations}

We now instantiate the abstract mechanism $\mcM$ from \cref{metaalg:auditor_unlearning} in two ways, and contrast both with a pairwise baseline adapted from DP auditing. Both instantiations use repeated runs of training and unlearning to calibrate a likelihood-based predictor, but differ in what they predict from the final unlearned model. Instantiation~I scores each candidate forget batch independently and is suitable when $m$ is too large to enumerate $\Sm$; Instantiation~II treats each balanced sign vector as a hypothesis and exploits cross-batch correlations, but is practical only when $|\Sm|$ is enumerable.

For a model $f$, write $f(x)_y$ for the probability assigned to label $y$ on input $x$ and $\phi(p) := \log\!\left(\tfrac{p}{1-p}\right)$ for the logit transform. We write $\mathcal{N}(s; \mu, \sigma^2) := \tfrac{1}{\sqrt{2\pi\sigma^2}} \exp\!\left(-\tfrac{(s-\mu)^2}{2\sigma^2}\right)$ for the Gaussian likelihood. Modelling scores as Gaussian is a choice for the attack only; \cref{lemma:lower_bound_overlap_score} requires no such assumption.

\paragraph{Instantiation~I: Batchwise inclusion/exclusion prediction.}
When $m$ is large, enumerating $\Sm$ is infeasible, so we score each batch independently. For every forget example, we fit two Gaussians to the logit scores collected across $\Gamma$ independent calibration runs $\mcU(\mcA(\mcD(S)), \mcD, \mcD_f(S))$ with $S \sim \Unif(\Sm)$: an \emph{in}-distribution conditioned on $S_j = +1$ and an \emph{out}-distribution conditioned on $S_j = -1$, in the spirit of LiRA \cite{carlini2022membershipinferenceattacksprinciples}.\footnote{See \cref{sec:appendix_logit_cross_entropy_comparison} for a comparison of logit score versus cross-entropy loss; the two yield similar values on most datasets.} At evaluation, for each batch we sum per-point log-likelihoods under the in- and out-distributions separately and take their difference to obtain a batch score $\Lambda_j$. We then predict $+1$ for the top $r/2$ batches by $\Lambda_j$, $-1$ for the bottom $r/2$, and $0$ for the rest. Calibration runs and evaluation runs are drawn independently, which is required for the audit's lower bound to be valid. See \cref{alg:instantiation_I} for pseudocode.

\paragraph{Instantiation~II: Joint sign-vector prediction.}
When $|\Sm|$ is enumerable, we treat each candidate sign vector as a hypothesis, capturing cross-batch correlations that batchwise scoring ignores. During calibration, for every candidate $\widetilde S \in \Sm$ and every forget example $(x, y)$, we fit a Gaussian to the logit scores collected from $\Gamma$ independent runs of $\mcU(\mcA(\mcD(\widetilde S)), \mcD, \mcD_f(\widetilde S))$. At evaluation, we score each candidate by the cumulative log-likelihood of the observed scores under its fitted Gaussians and return the highest-scoring candidate. Since the auditor commits to a full sign vector, $r = m$ here. See \cref{alg:instantiation_II} for pseudocode.

\paragraph{Pairwise auditor (baseline).}\label{par:pairwise_auditor}
As a baseline we adapt the pairwise distinguishability auditor of \citet{nasr2021adversaryinstantiationlowerbounds}, which distinguishes two unlearned models trained on different forget sets and converts true/false positive rates into an $\varepsilon$ lower bound via \cite[Thm.~2.1]{kairouz2015composition} with Clopper--Pearson confidence intervals; the full description and exact formula are deferred to \cref{sec:appendix_pairwise_auditor}.

\section{Datasets and auditor setup} \label{sec:datasets}

\paragraph{CIFAR-100~\citep{Krizhevsky09learningmultiple}.}\label{sec:cifar100_dataset}
We train a \texttt{TinyNetCIFAR100} model to ${\sim}55\%$ validation accuracy. The forget set $\mcD_f$ contains $10\%$ of training points ($4{,}500$); we use a \emph{uniform} split (sampled at random) and an \emph{adversarial} split (forget and retain drawn from largely disjoint label classes). Audit batch size is $B=1$ for uncertified and $B=750$ for certified algorithms; the rationale and full construction are in \cref{sec:appendix_cifar100_dataset,sec:instantiation_discussion_certified_uncertified}.

\paragraph{Shakespeare~\citep{shakespeare_gutenberg_complete_works}.}\label{sec:shakespeare_dataset}
Following \citet{mcmahan2017communication} we treat each role as a client, subsample $300$ roles, and partition at the role level into $33$ forget roles (${\sim}10\%$ of training characters) and $267$ retain roles. The model is a $2$-layer character-level LSTM reaching ${\sim}0.53$ accuracy; we evaluate only uncertified methods, using the batchwise inclusion/exclusion auditor with $B=400$. Full construction, score definition, and model details are in \cref{sec:appendix_shakespeare_dataset}.

\section{Unlearning Algorithms Audited}\label{sec:unlearning_algorithms}

For all runs below, $\varepsilon_{\mathrm{LB}}$ denotes the lower bound obtained from the mean aggregation of overlap scores in \cref{lemma:lower_bound_overlap_score}.\footnote{Mean and median aggregations show similar trends; a detailed comparison is given in \cref{sec:mean_median_comparison}.} For all experiments, the confidence parameter $\zeta$ is set to $0.05$.

We organise the audited algorithms into two groups: certified algorithms with formal $(\varepsilon,\delta)$ guarantees (\cref{sec:certified_unlearning_audited}), and heuristic algorithms without such guarantees (\cref{sec:uncertified_unlearning_audited}). \Cref{sec:unlearning_comparison} discusses them in detail, and full hyperparameters and per-method experimental details are deferred to \cref{sec:appendix_unlearning_algorithms}.

\subsection{Certified unlearning algorithms}\label{sec:certified_unlearning_audited}

\paragraph{Model clipping.}\label{sec:model_clipping}
We audit the certified unlearning algorithm of \citet{koloskova2025certifiedunlearningneuralnetworks}, which interleaves projected noisy gradient steps on an $\ell_2$-ball of radius $C_2$ with a noiseless retain-set fine-tuning phase. The number of noisy steps is set by Theorem~4.2 of \citet{koloskova2025certifiedunlearningneuralnetworks} to achieve the target $(\varepsilon,\delta)$ guarantee. Recalling the certified-method auditor configuration of \cref{sec:cifar100_dataset}---batch size $B=750$ ($m=6$ forget batches), joint sign-vector predictor (\cref{alg:instantiation_II}) with $\Gamma=50$ calibration rounds and $L=100$ runs---we sweep $\varepsilon$, $C_2$, and $\sigma$ across target privacy levels spanning five orders of magnitude (up to $\varepsilon=10^5$). At the prescribed stopping time of Theorem~4.2, the audit yields $\varepsilon_{\mathrm{LB}} \approx 0$ across the entire grid, consistent with the certification holding. As a diagnostic, we additionally compute $\varepsilon_{\mathrm{LB}}(t)$ at intermediate iterates---treating the algorithm as if it had terminated at step $t$---which remains near zero throughout the noisy phase, indicating that the audit gains no traction even between certification checkpoints. Full details are reported in \cref{sec:appendix_model_clipping_setup}.

% The full algorithm, hyperparameter grid, and per-step plots

\paragraph{Uncertified variant of model clipping.}
For this uncertified variant, we set $C_2=5$ and $\sigma=10^{-4}$ and truncate the noisy phase to $10$ epochs, after which the algorithm proceeds to its retain-set fine-tuning phase as in the certified setting. The ratio $C_2/\sigma$ is far larger than what Theorem~4.2 admits, so the algorithm is no longer certified; we report $\varepsilon_{\mathrm{LB}}$ as a function of step index $t$ with $L=500$ runs to obtain tighter bounds (\cref{fig:model_clipping}). Further details are given in \cref{sec:appendix_uncertified_model_clipping}. We compare both uniform and adversarial splits.

\paragraph{Rewind-to-delete.}\label{sec:rewind_to_delete}
Rewind-to-delete (R2D) provides certified unlearning in the full-batch setting by adding final-step Gaussian noise calibrated, per Theorem~3.1 of \citet{rewind2delete}, to make the trained and unlearned models $(\varepsilon, \delta)$-indistinguishable. The noise scale, sensitivity, and other choices are detailed in \cref{sec:appendix_rewind_to_delete}. We consider both uniform and adversarial splits, with plots in \cref{fig:r2d}.

% \paragraph{Rewind-to-delete.}\label{sec:rewind_to_delete}
% Rewind-to-delete (R2D) provides certified unlearning guarantees in the full-batch training setting, with the noise added in the final step set so as to meet the target $(\varepsilon,\delta)$ guarantee, as given in Theorem~3.1 of \citet{rewind2delete}. Since we consider $\varepsilon>1$, the noise parameter $\sigma$ must satisfy $\Phi\!\left(-\tfrac{\varepsilon \sigma}{\Delta} + \tfrac{\Delta}{2\sigma}\right) - e^{\varepsilon}\, \Phi\!\left(-\tfrac{\varepsilon \sigma}{\Delta} - \tfrac{\Delta}{2\sigma}\right) \le \delta$ where $\Phi$ is the Gaussian CDF and $\Delta = \|\theta'_T - \theta''_K\|$ is the sensitivity, with $\theta'_T$ the trained model and $\theta''_K$ the unlearned model obtained by loading the checkpoint $\theta'_{T-K}$ and then training for $K$ epochs on the retain set. Specific choices of $T$ and $K$ and the rationale for keeping them small are deferred to \cref{sec:appendix_rewind_to_delete}. We consider both uniform and adversarial splits, with plots in \cref{fig:r2d}.

\paragraph{Observations.} R2D attains tighter lower bounds than model clipping across the sweep, reflecting its single-shot full-batch noise calibration versus model clipping's per-step noise budget over many iterations. Both methods exhibit the uniform-vs.\ adversarial split asymmetry discussed in \cref{sec:unlearning_comparison}.

%\paragraph{Observations for model clipping and R2D.} R2D attains tighter lower bounds than model clipping, and the \emph{baseline pairwise auditor} performs well when bounds are small but saturates quickly otherwise. Uniform splits typically yield higher lower bounds than adversarial splits. We attribute this to the structure of the splits: uniform splits draw forget and retain from the same class distribution, so forget points are in-distribution relative to retain and share feature support, whereas adversarial splits assign disjoint classes and produce semantically separated forget points. Prior work shows that retain--forget feature entanglement leaves residual memorisation that unlearning cannot fully remove~\cite{cheng2026machineunlearningretainforgetentanglement}, and that semantically adjacent retain classes provide an exploitable attack channel~\cite{ebrahimpourboroojeny2025necessityoutputdistributionreweighting} -- both of which are operative in the uniform-split regime but not the adversarial one.

\begin{figure}[t]
	\centering
	\begin{subfigure}{0.30\textwidth}
		\centering
		\includegraphics[scale=0.25]{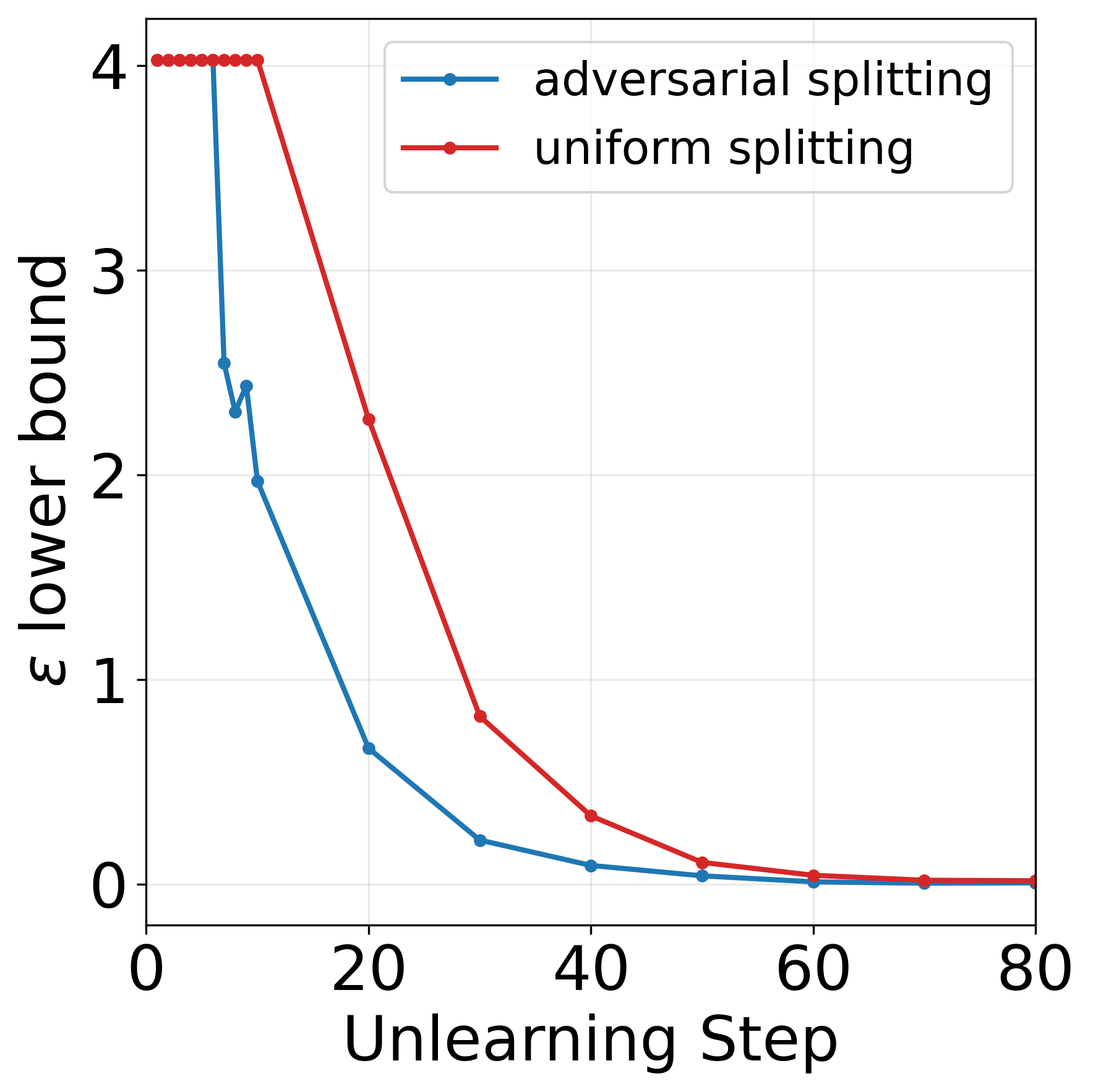}
		\caption{Uncertified variant of model clipping}
		\label{fig:model_clipping}
	\end{subfigure}
	\hfill
	\begin{subfigure}{0.30\textwidth}
		\centering
		\includegraphics[scale=0.25]{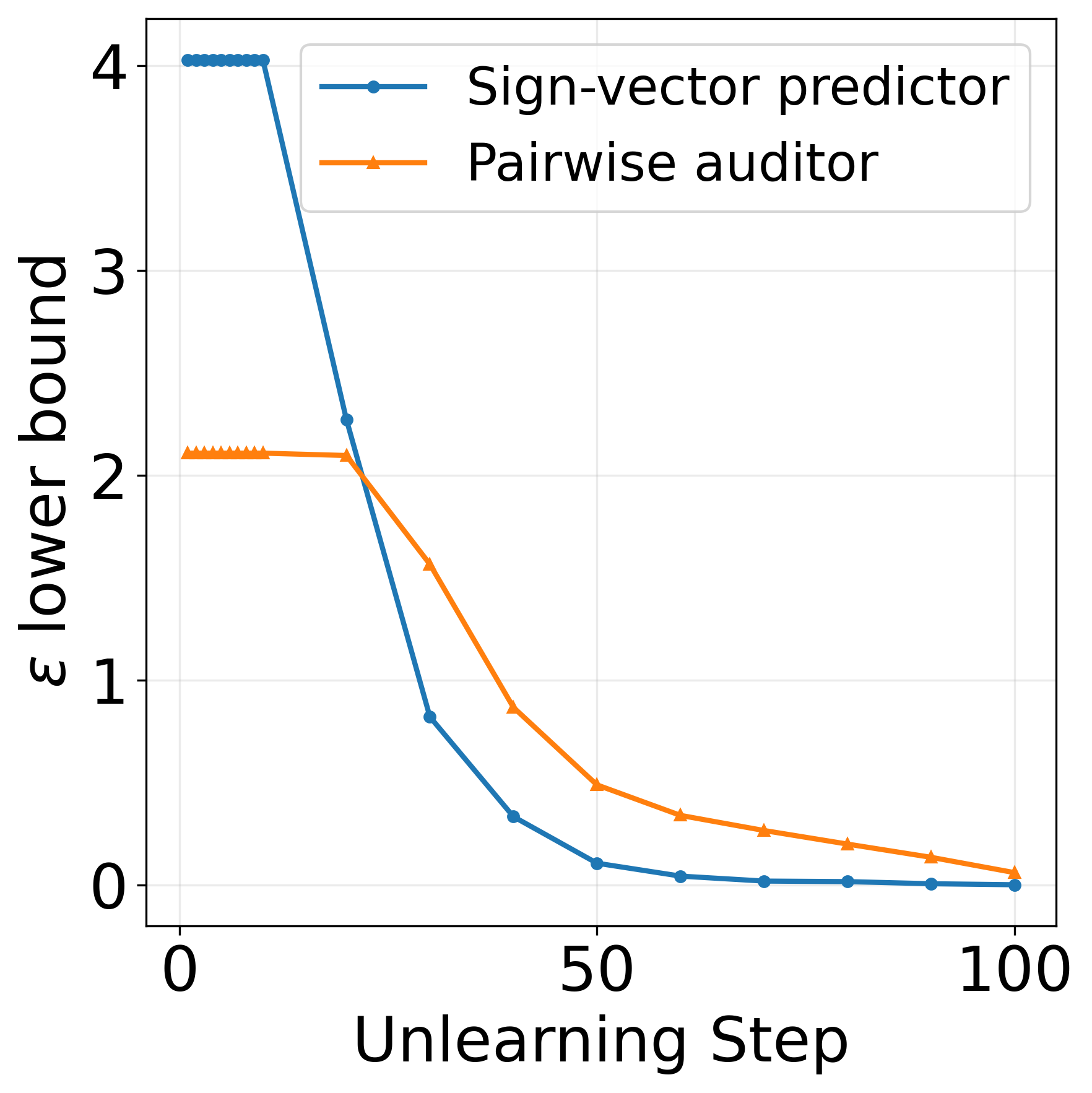}
		\caption{Comparison with pairwise auditor for uniform splitting}
		\label{fig:comp_pairwise}
	\end{subfigure}
	\begin{subfigure}{0.30\textwidth}
		\centering
		\includegraphics[scale=0.23]{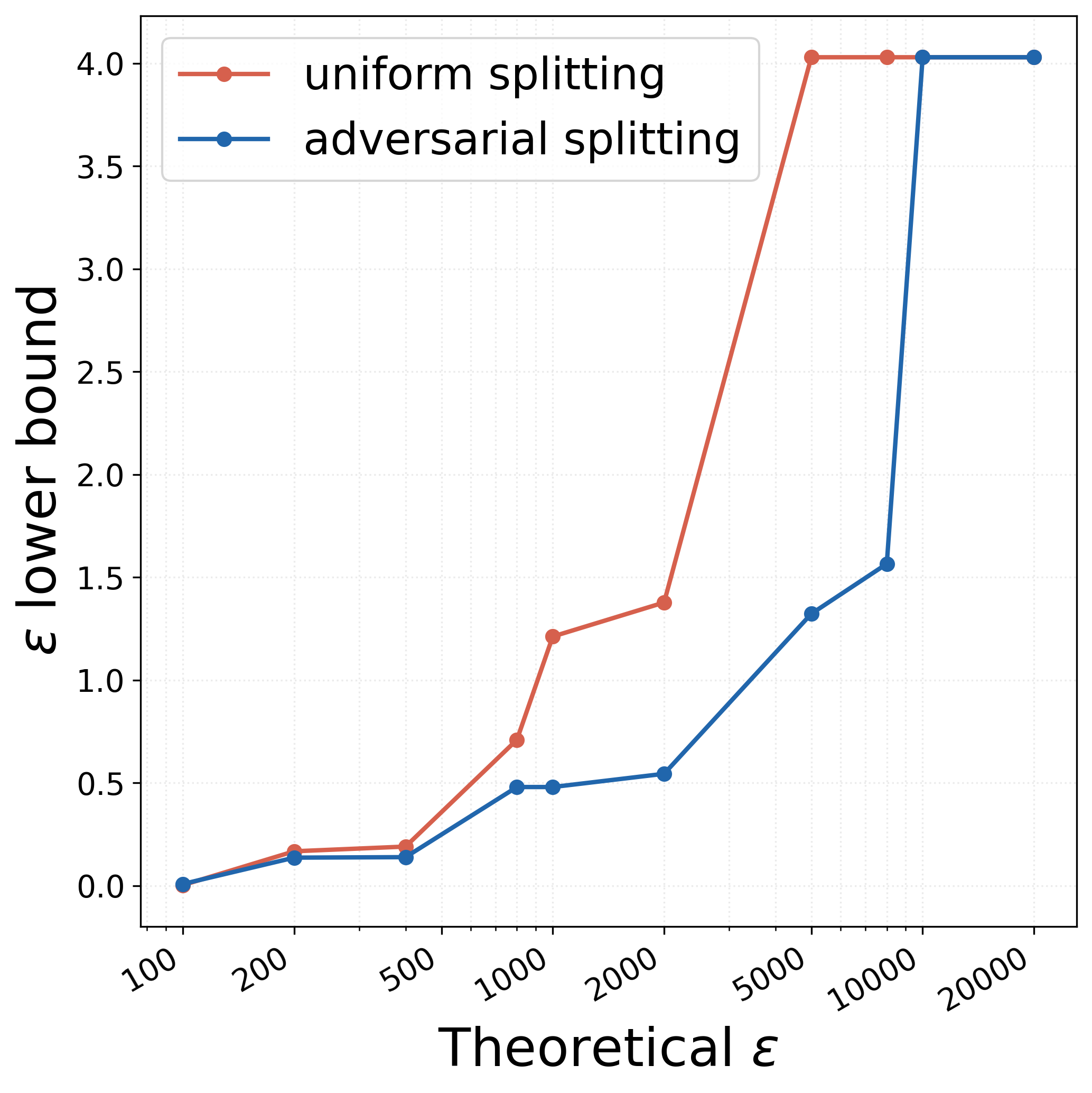}
		\caption{Rewind to delete for varying $\varepsilon$}
		\label{fig:r2d}
	\end{subfigure}
	\caption{Rewind to delete and model clipping audits for CIFAR100}
\end{figure}

\subsection{Uncertified unlearning algorithms}\label{sec:uncertified_unlearning_audited}

We audit four heuristic unlearning algorithms on both CIFAR-100 and Shakespeare: Hessian-based unlearning, interleaved descent--ascent retraining, ascent on the forget set, and pure fine-tuning on the retain set. Base-training settings and per-method learning-rate schedules are deferred to \cref{sec:appendix_uncertified_methods,sec:appendix_ida_retraining,sec:appendix_ascent_on_forget}. We choose support size $r=100$ for the Shakespeare dataset (with $m=400$ forget batches) and support size $r=3000$ for CIFAR-100 (with $m=4500$ batches). The number of calibration rounds $\Gamma$ is set to $50$ and $L=10$ evaluation runs.

\paragraph{Hessian-based unlearning.}\label{sec:hessian_unlearning}
We audit the Hessian-based unlearning algorithm of \citet{zhang2025certifiedunlearningdeepneural}, which performs a Newton-style update on the forget-set gradient followed by addition of Gaussian noise to every parameter. \citet{zhang2025certifiedunlearningdeepneural} prove certification under loss convexity; we test the algorithm in the nonconvex regime where this assumption is violated. The explicit update rule, the LiSSA approximation of the inverse-Hessian--vector product \cite{agarwallisasecondorder}, hyperparameter settings, and an ablation varying the final noise scale are deferred to \cref{sec:appendix_hessian_unlearning}.

\paragraph{Interleaved descent--ascent (IDA) retraining.}\label{sec:ida_retraining}
IDA alternates retain-set gradient descent with $\lambda$-weighted forget-gradient steps inserted after every $q$ retain steps, for the first $E_f$ epochs only; thereafter it reverts to plain retain descent. Update rule and per-dataset $(\lambda, q, E_f)$ are in \cref{sec:appendix_ida_retraining}.

% \paragraph{Interleaved descent--ascent (IDA) retraining.}\label{sec:ida_retraining}
% Interleaved descent--ascent (IDA) retraining performs retain-set gradient descent, with one \emph{interleaved} step inserted after every $q$ retain steps (i.e.\ once per cycle of $q+1$ updates). An interleaved step combines the retain gradient with a $\lambda$-weighted forget-gradient term, where $\lambda > 0$ controls the weight of the forget gradient relative to the retain gradient. Interleaving is applied only for the first $E_f$ epochs; thereafter the algorithm reverts to plain retain descent. Per-dataset values of $\lambda$, $q$, $E_f$, the learning-rate schedule, and the explicit update rule are deferred to \cref{sec:appendix_ida_retraining}.

\paragraph{Ascent on the forget set.}\label{sec:pure_ascent}
In this setup the unlearning phase consists of plain gradient ascent on the forget set, run for a fixed number of epochs (2 epochs for both datasets) and then followed by retain-set fine-tuning (analogous to the second phase of \nameref{sec:model_clipping}). Per-dataset learning rates, schedules, and the explicit update rule are given in \cref{sec:appendix_ascent_on_forget}.

\paragraph{Pure fine-tuning on the retain set.}\label{sec:pure_finetune}
The unlearning phase consists of pure fine-tuning on the retain set throughout. Equivalently, this is a special case of model clipping with clipping radius $C_2=\infty$ and noise scale $\sigma=0$. Learning-rate schedules match the retain-set fine-tuning phase of the corresponding base-training configuration (\cref{sec:appendix_uncertified_methods}).

\subsection{Comparison of lower bounds across methods}\label{sec:unlearning_comparison}

\Cref{tab:lower_bound_comparison_shakespeare,tab:lower_bound_comparison_cifar100} report the lower bounds in detail across hyperparameter configurations.

\paragraph{The Hessian-based method's certification does not transfer to the nonconvex regime.}\label{par:hessian_observation}
\citet{zhang2025certifiedunlearningdeepneural} prove that their Hessian-based update is $(\varepsilon,\delta)$-certified when the training loss is convex, an assumption violated by the TinyNetCIFAR100 model used here. Empirically, the audit yields $\varepsilon_{\mathrm{LB}} = 142.5$ on CIFAR-100 under uniform splitting---larger than \emph{any} bound we observe across heuristic methods, which carry no certification claim at all. We view this as the clearest demonstration of our auditor's diagnostic value: a method whose advertised certification rests on a precondition that silently fails at deployment is sharply flagged by the hypothesis test, while methods whose certifications hold in the actual training regime (model clipping, R2D) pass the same audit cleanly. The split-dependent gap on CIFAR-100 ($\varepsilon_{\mathrm{LB}} = 142.5$ uniform vs. $10$ adversarial) is consistent with the broader uniform-vs.-adversarial pattern discussed below.

\paragraph{Pairwise-auditor baseline.} For the pairwise auditor (\cref{par:pairwise_auditor}) with $L=10$, even $\mathrm{FPR}=\mathrm{FNR}=0$ yields a zero lower bound, due to the loose Clopper--Pearson confidence intervals.

\paragraph{IDA hyperparameters.} On CIFAR-100 under adversarial splitting, $\varepsilon_{\mathrm{LB}}$ for IDA increases monotonically with $q$ ($0 \to 18 \to 66$ for $q \in \{1, 2, 4\}$), indicating that unlearning quality degrades as the interleaving frequency drops.\footnote{Under uniform splitting we see a slight downward trend.} On Shakespeare, the dependence on $q$ is mild ($13.15 \to 13.78 \to 14.42$), but $\varepsilon_{\mathrm{LB}}$ decreases as either $\lambda$ or $E_f$ increases, as one would expect: larger forget-gradient weight and a longer interleaving phase both translate to more aggressive forgetting and a looser auditor signal.

\paragraph{Forget ascent and retain fine-tuning.} Ascent on the forget set is essentially undetectable on Shakespeare ($\varepsilon_{\mathrm{LB}} = 0.17$) but readily audited on CIFAR-100 ($\mathrm{U} = 79.5$, $\mathrm{A} = 86$). Pure retain fine-tuning shows similarly large CIFAR-100 bounds in both splits ($\mathrm{U} = 81.5$, $\mathrm{A} = 84$).

\paragraph{Uniform vs.\ adversarial splits.} Across both certified and uncertified methods, $\varepsilon_{\mathrm{LB}}$ under uniform splitting typically exceeds adversarial (with pure retain fine-tuning and forget ascent on CIFAR-100 as mild exceptions). Uniform splits share feature support between retain and forget, leaving residual memorisation~\cite{cheng2026machineunlearningretainforgetentanglement} and exploitable semantic adjacency~\cite{ebrahimpourboroojeny2025necessityoutputdistributionreweighting}; adversarial splits do not. The pairwise-auditor baseline performs well when bounds are small but saturates quickly otherwise.

% \paragraph{Uniform vs.\ adversarial splits.} As in the certified setting (\cref{sec:certified_unlearning_audited}), lower bounds under uniform splitting typically exceed those under adversarial splitting, with pure retain-set fine-tuning and forget-set ascent on CIFAR-100 as mild exceptions where adversarial is mildly higher.

\begin{table}[t]
	\centering
	\renewcommand{\arraystretch}{0.92}
	\begin{subtable}[t]{0.48\textwidth}
		\centering
		\scriptsize
		\setlength{\tabcolsep}{3pt}
		\begin{tabular}{llc}
			\toprule
			Setting & Parameters & $\varepsilon_{\mathrm{LB}}$ \\
			\midrule
			\multirow{3}{*}{\shortstack[l]{Varying $q$\\(IDA)}}
			& $q=1$, $\lambda=0.5$, $E_f=4$ & 13.15 \\
			& $q=4$, $\lambda=0.5$, $E_f=4$ & 13.78 \\
			& $q=9$, $\lambda=0.5$, $E_f=4$ & 14.42 \\
			\midrule
			\multirow{4}{*}{\shortstack[l]{Varying\\$(\lambda, E_f)$,\\$q=1$}}
			& $\lambda=0.5$, $E_f=4$   & 13.15 \\
			& $\lambda=0.5$, $E_f=7$ & 10.38 \\
			& $\lambda=1.5$, $E_f=4$   & 9.97 \\
			& $\lambda=1.5$, $E_f=7$     & 0.17  \\
			\midrule
			\multirow{3}{*}{\shortstack[l]{Method\\comp.}}
			& Retain fine-tune  & 16.06 \\
			& Forget ascent     & 0.17  \\
			& Hessian-based     & 4.18  \\
			\bottomrule
		\end{tabular}
		\caption{Shakespeare. Nearly every setting has retain/forget accuracy averaged around 0.56 and test accuracy averaged around 0.53. Ascent, Finetune, and IDA all use total 8 unlearning epochs.}
		\label{tab:lower_bound_comparison_shakespeare}
	\end{subtable}
	\hfill
	\begin{subtable}[t]{0.50\textwidth}
		\centering
		\scriptsize
		\setlength{\tabcolsep}{2pt}
		\begin{tabular}{llcccc}
			\toprule
			& & \multicolumn{2}{c}{$\varepsilon_{\mathrm{LB}}$} & \multicolumn{2}{c}{R/F/T accuracy} \\
			\cmidrule(lr){3-4} \cmidrule(lr){5-6}
			Setting & Params & U & A & U & A \\
			\midrule
			\multirow{3}{*}{\shortstack[l]{IDA}}
			& $q=1, E_f = 10$ & 100.5 & 0.0   & 0.91/0.59/0.54 & 0.92/0.02/0.51 \\
			& $q=2, E_f = 10$ & 100   & 18  & 0.92/0.59/0.54 & 0.92/0.027/0.51 \\
			& $q=4, E_f = 10$ & 95.5  & 66  & 0.92/0.59/0.54 & 0.92/0.02/0.51 \\
			\midrule
			\multirow{3}{*}{\shortstack[l]{Method\\comp.}}
			& Pure Retain FT & 81.5  & 84   & 0.93/0.59/0.53 & 0.94/0.03/0.51 \\
			& Forget ascent  & 79.5  & 86   & 0.92/0.59/0.54 & 0.94/0.025/0.51 \\
			& Hessian        & 142.5 & 10   & 0.69/0.67/0.51 & 0.55/0.09/0.37 \\
			\bottomrule
		\end{tabular}
		\caption{CIFAR-100. Ascent, Finetune, and IDA all use total of 50 unlearning epochs \textnormal{U = uniform split, A = adversarial split; R/F/T = retain/forget/test; FT = Finetuning}}
		\label{tab:lower_bound_comparison_cifar100}
	\end{subtable}
	\caption{$\varepsilon$ lower bounds across Shakespeare and CIFAR-100.}
	\label{tab:lower_bound_comparison_combined}
\end{table}

\section{Unlearning on convex losses}

The certified algorithms audited above yield very loose bounds ($\varepsilon_{\mathrm{LB}} \approx 0$ for model clipping even at target $\varepsilon = 10^5$, and tighter but still weak for R2D), raising the question of how much of this gap reflects auditor slack versus genuine unlearning. To probe this, we turn to \emph{convex} losses, where the mechanism is analytically tractable and a tighter correspondence between $\varepsilon$ and $\varepsilon_{\mathrm{LB}}$ is expected. We adopt the stronger certified-unlearning notion of \citet{sekharimachineunlearningconvex} (\cref{def:convex_unlearning}), which compares the unlearned output to the same mechanism invoked with an empty forget set; since this compares two outputs of the same procedure, no transitivity step is needed (unlike \cref{sec:redn_local_dp}), so we audit at the same $\delta$ as the certified guarantee.

We audit two mechanisms: the \emph{Perturbed Newton} unlearning of \citet[Algorithm~1]{sekharimachineunlearningconvex} on a bounded cubic loss in dimension $d=1$ (and a logistic loss in the appendix), and a simple \emph{output perturbation} mechanism that clips the trained model and adds Gaussian noise. The auditor adapts the pairwise baseline of \cref{par:pairwise_auditor} to fit Gaussians directly on model weights; we also report $\varepsilon_{\mathrm{LB}}$ from instantiating \cref{metaalg:auditor_unlearning} with $m=2$, $r=2$ under $\delta=0$. Full pipeline, losses, and hyperparameters are in \cref{sec:appendix_convex_unlearning,sec:appendix_output_perturbation}.

\Cref{tab:combined_unlearning_lower_bounds} reports the bounds. Both saturate near $10.19$ at $2L=200{,}000$ test points; the logistic loss yields a $0$ bound, discussed in \cref{sec:appendix_convex_unlearning}. As expected, the gap between $\varepsilon_{\mathrm{LB}}$ and the certified $\varepsilon$ is much smaller here than for the nonconvex algorithms of \cref{sec:certified_unlearning_audited}, confirming that the convex setting admits a tighter audit.

\begin{table*}[t]
	\centering

	\begin{subtable}[t]{0.48\textwidth}
		\centering
		\scriptsize
		\begin{tabular}{cccc}
			\toprule
			$\varepsilon$ & $\sigma$ & $\varepsilon^{(p)}_{\text{LB}}$ & $\varepsilon_{\text{LB}}$ \\
			\midrule
			0.1   & 0.6197 & 0.0     & 0.0     \\
			%0.5   & 0.2044 & 0.0     & 0.0     \\
			%1.0   & 0.1220 & 0.0     & 0.00  \\
			%2.0   & 0.0725 & 0.1410  & 0.0     \\
			5.0   & 0.0370 & 0.2902  & 0.01406 \\
			%10.0  & 0.0227 & 0.4879  & 0.0552  \\
			%20.0  & 0.0143 & 0.7970  & 0.1537  \\
			50.0  & 0.0081 & 1.4727  & 0.4827  \\
			100.0 & 0.0054 & 2.3322  & 1.0504  \\
			200.0 & 0.0036 & 3.7350  & 2.1455  \\
			400.0 & 0.0025 & 6.2377  & 4.3663  \\
			600.0 & 0.0020 & 8.5237  & 6.7425  \\
			800.0 & 0.0017 & 10.1983 & 9.4159  \\
			\bottomrule
		\end{tabular}
		\caption{Convex cubic loss (Perturbed Newton).}
		\label{tab:unlearning_lower_bounds_convex_loss}
	\end{subtable}
	\hfill
	\begin{subtable}[t]{0.48\textwidth}
		\centering
		\scriptsize
		\begin{tabular}{cccc}
			\toprule
			$\varepsilon$ & $\sigma$ & $\varepsilon^{(p)}_{\text{LB}}$ & $\varepsilon_{\text{LB}}$ \\
			\midrule
			0.1   & 1.9084 & 0.0191  & 0.0     \\
			0.5   & 0.6294 & 0.2040  & 0.0027  \\
			1.0   & 0.3756 & 0.3901  & 0.0307  \\
			2.0   & 0.2233 & 0.7029  & 0.1120  \\
			5.0   & 0.1139 & 1.4534  & 0.4547  \\
			10.0  & 0.0700 & 2.4962  & 1.1518  \\
			20.0  & 0.0441 & 4.4305  & 2.6972  \\
			50.0  & 0.0249 & 9.5260  & 8.0904  \\
			100.0 & 0.0166 & 10.1983 & 11.1039 \\
			\bottomrule
		\end{tabular}
		\caption{Output perturbation.}
		\label{tab:epsilon_lower_bounds_output_perturbation}
	\end{subtable}

	\caption{Lower bounds for Perturbed Newton unlearning on the cubic loss (left) and for output perturbation on a linear-regression task (right). $\varepsilon^{(p)}_{\text{LB}}$ and $\varepsilon_{\text{LB}}$ denote the bounds from the \emph{pairwise auditor} (\cref{par:pairwise_auditor}) and the \emph{joint sign-vector} predictor (with $m=2$, $r=2$), respectively.}
	\label{tab:combined_unlearning_lower_bounds}
\end{table*}

%200.0 & 0.0112 & 10.1983 & 11.1039 \\
%400.0 & 0.0077 & 10.1983 & 11.1039 \\
%600.0 & 0.0062 & 10.1983 & 11.1039 \\
%800.0 & 0.0053 & 10.1983 & 11.1039 \\

\section{Conclusion}

We propose an auditor that lower-bounds the unlearning parameter $\varepsilon$ via membership inference, demonstrating a sharp separation between certified and uncertified unlearning algorithms: certified methods yield very small $\varepsilon_{\mathrm{LB}}$, whereas heuristic methods regularly yield $\varepsilon_{\mathrm{LB}}$ in the tens or higher. Our analysis is currently restricted to $\delta = 0$; extending the audit to $\delta > 0$, tightening the lower bounds via stronger membership-inference attacks, and designing certified unlearning algorithms whose nominal $\varepsilon$ is provably close to $\varepsilon_{\mathrm{LB}}$ are natural directions for future work.

\bibliographystyle{plainnat}
\bibliography{refs}

%%%%%%%%%%%%%%%%%%%%%%%%%%%%%%%%%%%%%%%%%%%%%%%%%%%%%%%%%%%%%%%%%%%%%%%%%%%%%%%
%%%%%%%%%%%%%%%%%%%%%%%%%%%%%%%%%%%%%%%%%%%%%%%%%%%%%%%%%%%%%%%%%%%%%%%%%%%%%%%
% APPENDIX
%%%%%%%%%%%%%%%%%%%%%%%%%%%%%%%%%%%%%%%%%%%%%%%%%%%%%%%%%%%%%%%%%%%%%%%%%%%%%%%
%%%%%%%%%%%%%%%%%%%%%%%%%%%%%%%%%%%%%%%%%%%%%%%%%%%%%%%%%%%%%%%%%%%%%%%%%%%%%%%
\newpage
\appendix
\onecolumn
\section{Auditor algorithms: pseudocode}\label{sec:appendix_auditor_algorithms}

This appendix gives pseudocode for the meta-algorithm of \cref{sec:algorithm_audit} and the two auditor instantiations (\emph{batchwise inclusion/exclusion} and \emph{joint sign vector prediction}) of \cref{sec:auditor_instantiations}; see \cref{metaalg:auditor_unlearning,alg:instantiation_I,alg:instantiation_II}.

\subsection{Meta-algorithm}\label{sec:appendix_meta_algorithm}

\begin{algorithm}[H]
\caption{Meta-algorithm for auditor}
\label{metaalg:auditor_unlearning}
\begin{algorithmic}[1]
\REQUIRE Batched dataset $\mcD \in \mcB^n$ partitioned as $\mcD = \mcD_r \cup \mcD_f$ with $|\mcD_f|=m$
\REQUIRE Algorithms $(\mcA,\mcU)$, reported support size $r$, and number of runs $L$
\FOR{$\ell = 1,2,\ldots,L$}
    \STATE Sample $S^{(\ell)} \sim \Unif(\Sm)$ and unlearnt model $f_u^{(\ell)} \gets \mcU\big(\mcA(\mcD(S^{(\ell)})), \, \mcD(S^{(\ell)}), \, \mcD_f(S^{(\ell)})\big)$.
    \STATE \emph{Induced mechanism $\mcM(S^{(\ell)})$}: From $f_u^{(\ell)}$, output $\Shat^{(\ell)} \in \{-1,0,+1\}^{m}$ with exactly $r/2$ entries $+1$ and $r/2$ entries $-1$.
    \STATE Compute overlap score $V^{(\ell)} := \sum_{j=1}^{m} \max\{0, \, \Shat_j^{(\ell)} S_j^{(\ell)}\}$.
\ENDFOR
\STATE \textbf{Return} summary statistics of $\{V^{(\ell)}\}_{\ell=1}^{L}$ (e.g., mean or median).
\end{algorithmic}
\end{algorithm}

\subsection{Batchwise inclusion/exclusion prediction (Instantiation~I)}\label{sec:appendix_batchwise_inclusion_exclusion}

\begin{algorithm}[H]
	\caption{Auditor Instantiation I: Batchwise inclusion/exclusion prediction}
	\label{alg:instantiation_I}
	\begin{algorithmic}[1]
		\REQUIRE Batched dataset \(\mcD=\mcD_r\cup\mcD_f\) with \(|\mcD_f|=m\); algorithms \((\mcA,\mcU)\); calibration rounds \(\Gamma\); support size \(r\)
		\ENSURE Mechanism \(\mcM:\Sm\to\{-1,0,+1\}^m\)
		
		\vspace{0.3em}
		\STATE \textsc{// Calibration phase \hfill (run once, offline)}
		\FOR{\(\gamma=1,\dots,\Gamma\)}
		\STATE \(S^{(\gamma)}\sim \Unif(\mcS_m)\), \quad \(\mcD^{(\gamma)}\gets\mcD_r\cup \mcD_f(S^{(\gamma)})\)
		\STATE \(f^{(\gamma)} \gets \mcU\big(\mcA(\mcD^{(\gamma)}),\,\mcD^{(\gamma)},\,\mcD_f(S^{(\gamma)})\big)\)
		\STATE \(s^{(\gamma)}_{x,y} \gets \phi\!\left(f^{(\gamma)}(x)_y\right)\) \quad for all \((x,y)\in\mcD_f\)
		\ENDFOR
		\STATE For each $j\in[m]$ and \((x,y)\in\mcD_{f,j}\), fit Gaussians:
		\[
		(\mu^{\pm}_{x,y},\,\sigma^{\pm}_{x,y})
		\;\gets\; \mathrm{fit}\!\left(\bigl\{s^{(\gamma)}_{x,y} : S^{(\gamma)}_{j}=\pm 1\bigr\}_{\gamma=1}^{\Gamma}\right)
		\]
		
		\vspace{0.4em}
		\STATE \textsc{// Mechanism \(\mcM(S)\) \hfill (applied at evaluation)}
		\STATE \quad \(f \gets \mcU\big(\mcA(\mcD(S)),\,\mcD(S),\,\mcD_f(S)\big)\), \quad \(s_{x,y} \gets \phi\!\left(f(x)_y\right)\)
		\STATE \quad For each batch \(j\in[m]\), compute
		\[
		\Lambda_j \;\gets\; \sum_{(x,y)\in\mcD_{f,j}}
		\log\frac{\mcN\!\left(s_{x,y};\,\mu^{+}_{x,y},\,\sigma^{+}_{x,y}\right)}
		{\mcN\!\left(s_{x,y};\,\mu^{-}_{x,y},\,\sigma^{-}_{x,y}\right)}
		\]
		\STATE \quad \(\Shat_j \gets +1\) for top-\(r/2\) values of \(\Lambda_j\); \(\;-1\) for bottom-\(r/2\); \(\;0\) otherwise
		\STATE \quad \textbf{return} \(\Shat\in\{-1,0,+1\}^m\)
	\end{algorithmic}
\end{algorithm}

\subsection{Pairwise auditor (baseline)}\label{sec:appendix_pairwise_auditor}

\citet{nasr2021adversaryinstantiationlowerbounds} audit DP by having the adversary distinguish two models trained on neighbouring datasets. We adapt this baseline to unlearning by distinguishing the two unlearned models $f_u^{(i)} := \mcU\bigl(\mcA(\mcD_r \cup \mcD_f^{(i)}),\, \mcD_r \cup \mcD_f^{(i)},\, \mcD_f^{(i)}\bigr)$ for $i \in \{1, 2\}$, corresponding to two distinct forget sets $\mcD_f^{(1)}, \mcD_f^{(2)}$. Calibration and prediction are identical to Instantiation~II with $m = 2$ and $\mcD_{f, j} = \mcD_f^{(j)}$; only the conversion to an $\varepsilon$ lower bound differs. From the true and false positive rates, the lower bound is computed directly from the confusion matrix using \cite[Theorem 2.1]{kairouz2015composition} with Clopper--Pearson confidence intervals, rather than via \cref{lemma:lower_bound_overlap_score}. For a given confidence level $\zeta$ and appropriate $\delta$, the bound is
\begin{equation}\label{eq:pairwise_auditor_lb}
\varepsilon_{\mathrm{LB}} = \max\!\left\{\log \frac{1 - \delta - \mathrm{FP}^{\mathrm{high}}}{\mathrm{FN}^{\mathrm{high}}},\;\; \log \frac{1 - \delta - \mathrm{FN}^{\mathrm{high}}}{\mathrm{FP}^{\mathrm{high}}}\right\},
\end{equation}
where $\mathrm{FP}^{\mathrm{high}}$ and $\mathrm{FN}^{\mathrm{high}}$ are the upper Clopper--Pearson bounds on the false-positive and false-negative rates. \cref{tab:combined_unlearning_lower_bounds} compares this baseline with the bounds from \cref{lemma:epsilon_lower_bound}.

\subsection{Joint sign-vector prediction (Instantiation~II)}\label{sec:appendix_joint_sign_vector_prediction}

\begin{algorithm}[H]
	\caption{Auditor Instantiation II: Joint sign-vector prediction}
	\label{alg:instantiation_II}
	\begin{algorithmic}[1]
		\REQUIRE Batched dataset \(\mcD=\mcD_r\cup\mcD_f\) with \(|\mcD_f|=m\); algorithms \((\mcA,\mcU)\); calibration rounds \(\Gamma\)
		\ENSURE Mechanism \(\mcM:\Sm\to\Sm\)
		
		\vspace{0.3em}
		\STATE \textsc{// Calibration phase \hfill (run once, offline)}
		\FORALL{\(\widetilde S \in \mcS_m\)}
		\FOR{\(\gamma=1,\dots,\Gamma\)}
		\STATE \(f^{(\gamma)} \gets \mcU\big(\mcA(\mcD(\widetilde S)),\,\mcD(\widetilde S),\,\mcD_f(\widetilde S)\big)\)
		\STATE \(s^{(\gamma)}_{x,y}(\widetilde S) \gets \phi\!\left(f^{(\gamma)}(x)_y\right)\) \quad for all \((x,y)\in\mcD_f\)
		\ENDFOR
		\STATE For each \((x,y)\in\mcD_f\), fit Gaussian:
		\[
		(\mu_{x,y}(\widetilde S),\,\sigma_{x,y}(\widetilde S))
		\;\gets\; \mathrm{fit}\!\left(\bigl\{s^{(\gamma)}_{x,y}(\widetilde S)\bigr\}_{\gamma=1}^{\Gamma}\right)
		\]
		\ENDFOR
		
		\vspace{0.4em}
		\STATE \textsc{// Mechanism \(\mcM(S)\) \hfill (applied at evaluation)}
		\STATE \quad \(f \gets \mcU\big(\mcA(\mcD(S)),\,\mcD(S),\,\mcD_f(S)\big)\), \quad \(s_{x,y} \gets \phi\!\left(f(x)_y\right)\)
		\STATE \quad For each \(\widetilde S \in \mcS_m\), compute
		\[
		\Lambda(\widetilde S) \;\gets\; \sum_{(x,y)\in\mcD_f}
		\log \mcN\!\left(s_{x,y};\,\mu_{x,y}(\widetilde S),\,\sigma_{x,y}(\widetilde S)\right)
		\]
		\STATE \quad \textbf{return} \(\Shat \;\gets\; \arg\max_{\widetilde S\in\mcS_m} \Lambda(\widetilde S)\)
	\end{algorithmic}
\end{algorithm}

\section{Proof of \cref{lemma:lower_bound_overlap_score}}

We now restate \cref{lemma:lower_bound_overlap_score} and prove it below.

\begin{lemma*}[Restatement of \cref{lemma:lower_bound_overlap_score}]
	\lemmalowerbounddeltazero[_restated]
\end{lemma*}

\begin{proof}
We first establish the pointwise bound
\begin{equation}\label{eq:pw}
\Pr\bigl[V^{(\ell)}=u\bigr]\;\le\;\pi_{\varepsilon}(u)
\qquad\text{for every }u\in\{0,1,\dots,r\}.
\end{equation}
Fix $\ell$ and write $S=S^{(\ell)}$, $\Shat=\Shat^{(\ell)}$, $V=V^{(\ell)}$.

\paragraph{Posterior bound from $\varepsilon$-LDP.}
Fix any $t\in\mathcal{T}$ in the support of $\Shat$ and any $s\in\Sm$. Pure
$\varepsilon$-LDP gives
$\Pr[\Shat=t\mid S=s]\le e^{\varepsilon}\Pr[\Shat=t\mid S=s']$ for every
$s'\in\Sm$, hence
$\Pr[\Shat=t\mid S=s']\ge e^{-\varepsilon}\Pr[\Shat=t\mid S=s]$.
Summing over $s'\in\Sm$ and using the uniform prior on $\Sm$,
\[
M'\Pr[\Shat=t]
=\sum_{s'\in\Sm}\Pr[\Shat=t\mid S=s']
\;\ge\;\Pr[\Shat=t\mid S=s]\bigl(1+(M'-1)e^{-\varepsilon}\bigr).
\]
Bayes' rule then yields the pointwise posterior bound
\begin{equation}\label{eq:posterior}
\Pr[S=s\mid \Shat=t]
\;=\;\frac{\Pr[\Shat=t\mid S=s]}{M'\Pr[\Shat=t]}
\;\le\;\frac{1}{1+(M'-1)e^{-\varepsilon}}
\;=\;\frac{e^{\varepsilon}}{e^{\varepsilon}+M'-1}.
\end{equation}

\paragraph{Counting argument.}
For $(s,t)\in\Sm\times\mathcal{T}$, let
$\alpha_1:=|\{j:t_j=+1,s_j=+1\}|$ and $\alpha_2:=|\{j:t_j=-1,s_j=-1\}|$,
both in $\{0,\dots,r/2\}$. Since $t_j\ne 0$ on exactly $r$ coordinates,
\[
V(s,t)\;=\;\sum_{j=1}^{m}\max\{0,s_jt_j\}\;=\;\alpha_1+\alpha_2.
\]
For each fixed $t\in\mathcal{T}$, the number of $s\in\Sm$ with $V(s,t)=u$ is
exactly $N_u$: pick the $\alpha_1$ agreement positions among the $r/2$
coordinates where $t_j=+1$ ($\binom{r/2}{\alpha_1}$ ways) and the $\alpha_2$
agreement positions among the $r/2$ coordinates where $t_j=-1$
($\binom{r/2}{\alpha_2}$ ways). The total $+1$/$-1$ imbalance of $s$ is then
$2(\alpha_1-\alpha_2)+2\beta-(m-r)$, where $\beta$ is the number of $+1$
entries of $s$ on the zero-block of $t$. The constraint $s\in\Sm$ forces
this imbalance into $\{0,1\}$, whose unique integer solution in $\beta$ is
$\beta=\lceil(m-r)/2\rceil-(\alpha_1-\alpha_2)$, contributing the
$\binom{m-r}{\lceil(m-r)/2\rceil-(\alpha_1-\alpha_2)}$ factor.

\paragraph{Assembling \eqref{eq:pw}.}
Conditioning on $\Shat=t$ and applying \eqref{eq:posterior} to each of the
$N_u$ posterior summands,
\[
\Pr[V=u\mid \Shat=t]
\;=\;\sum_{s:\,V(s,t)=u}\Pr[S=s\mid \Shat=t]
\;\le\;N_u\cdot\frac{e^{\varepsilon}}{e^{\varepsilon}+M'-1}
\;=\;\pi_{\varepsilon}(u).
\]
Marginalizing over $\Shat$ gives $\Pr[V=u]\le\pi_{\varepsilon}(u)$, proving
\eqref{eq:pw}. Summing over $u\ge v$,
\begin{equation}\label{eq:tail}
\Pr[V\ge v]\;\le\;P_{\varepsilon}(v).
\end{equation}

\paragraph{Median bound \eqref{eq:lb-median}.}
The event
$\mathrm{Median}(\{V^{(\ell)}\}_{\ell=1}^{L})\ge v$ implies the existence of
some $J\subseteq[L]$ with $|J|=\lceil L/2\rceil$ such that $V^{(\ell)}\ge v$
for every $\ell\in J$. By independence and a union bound over the
$\binom{L}{\lfloor L/2\rfloor}$ such subsets, together with \eqref{eq:tail},
\[
\Pr\!\left(\mathrm{Median}\!\left(\{V^{(\ell)}\}_{\ell=1}^{L}\right)\ge v\right)
\;\le\;\binom{L}{\lceil L/2\rceil}P_{\varepsilon}(v)^{\lceil L/2\rceil}.
\]

\paragraph{Chernoff bound \eqref{eq:lb-mean}.}
Fix $\lambda\ge 0$. By \eqref{eq:pw},
\[
\E\!\left[e^{\lambda V^{(\ell)}}\right]
\;=\;\sum_{u=0}^{r}e^{\lambda u}\Pr[V^{(\ell)}=u]
\;\le\;\sum_{u=0}^{r}e^{\lambda u}\pi_{\varepsilon}(u).
\]
By independence across $\ell$ and Markov's inequality applied to
$e^{\lambda\sum_{\ell}V^{(\ell)}}$,
\[
\Pr\!\left[\frac{1}{L}\sum_{\ell=1}^{L}V^{(\ell)}\ge v\right]
\;\le\;e^{-\lambda Lv}
\Bigg(\sum_{u=0}^{r}e^{\lambda u}\pi_{\varepsilon}(u)\Bigg)^{\!L}
\;=\;\exp\!\left(L\log\!\Big(\sum_{u=0}^{r}e^{\lambda u}\pi_{\varepsilon}(u)\Big)-\lambda Lv\right).
\]
Optimizing over $\lambda\ge 0$ yields \eqref{eq:lb-mean}.
\end{proof}
We now state a corollary showing that the upper bounds for median and mean statistic proven above is monotonic in $\varepsilon$ and thus, to reject the hypothesis $\varepsilon \leq \varepsilon_{\text{LB}}$, it is sufficient to evaluate the bound at $\varepsilon= \varepsilon_{\text{LB}}$ for $\delta< \frac{M'}{K}$. %One may note that for most regimes that we operate in this bound naturally holds true.

\begin{remark}
    For uncertified algorithms, we choose $m=4500$, $r=3000$ and thus one can see that $K<M'$ naturally holds true.
\end{remark}

\begin{corollary}[Monotonicity of audit bounds]\label{cor:bound_monotonicity}
For $\delta = 0$, the right-hand sides of \cref{eq:lb-mean} and \cref{eq:lb-median} in \cref{lemma:lower_bound_overlap_score} are monotone non-decreasing in $\varepsilon$.
\end{corollary}

\begin{proof}
For $\delta = 0$, the pointwise bound factorises as $\pi_\varepsilon(u) = C(u) \cdot g(\varepsilon)$, where
\[
C(u) := \sum_{\substack{\alpha_1, \alpha_2 \in \{0, \dots, r/2\} \\ \alpha_1 + \alpha_2 = u}}
\binom{m-r}{\lceil(m-r)/2\rceil - (\alpha_1 - \alpha_2)}
\binom{r/2}{\alpha_1}\binom{r/2}{\alpha_2} \;\ge\; 0
\]
is independent of $\varepsilon$, and
\[
g(\varepsilon) := \frac{e^\varepsilon}{e^\varepsilon + M' - 1} = \frac{1}{1 + (M' - 1)\,e^{-\varepsilon}}.
\]
Since $M' \ge 1$, the function $g$ is strictly increasing in $\varepsilon$, and hence $\pi_\varepsilon(u)$ is monotone non-decreasing in $\varepsilon$ for every fixed $u \in \{0, 1, \dots, r\}$.

\textbf{Median bound (\cref{eq:lb-median}).} The upper tail $P_\varepsilon(v) = \sum_{u = \lceil v \rceil}^{r} \pi_\varepsilon(u)$ is a non-negative sum of monotone non-decreasing functions, and is therefore itself monotone non-decreasing in $\varepsilon$. Raising to the power $\lceil L/2 \rceil$ and multiplying by the positive constant $\binom{L}{\lfloor L/2 \rfloor}$ both preserve monotonicity.

\textbf{Mean bound (\cref{eq:lb-mean}).} Fix any $\lambda \ge 0$. The moment generating function $\sum_{u=0}^{r} e^{\lambda u} \pi_\varepsilon(u)$ is a non-negative linear combination of $\{\pi_\varepsilon(u)\}_u$, hence monotone non-decreasing in $\varepsilon$. Composing with $\log$ (monotone), multiplying by $L > 0$, subtracting the $\varepsilon$-independent term $\lambda L v$, and exponentiating all preserve monotonicity. Therefore, for each fixed $\lambda \ge 0$, the function inside the infimum is monotone non-decreasing in $\varepsilon$. A pointwise infimum of monotone non-decreasing functions is itself monotone non-decreasing: for $\varepsilon_1 \le \varepsilon_2$ and any $\lambda$, $f_\lambda(\varepsilon_1) \le f_\lambda(\varepsilon_2)$, so taking the infimum over $\lambda$ on both sides yields $\inf_\lambda f_\lambda(\varepsilon_1) \le \inf_\lambda f_\lambda(\varepsilon_2)$. This completes the proof.
\end{proof}

\section{A discussion on setting $\delta>0$ in the auditing lower bound }
\label{sec:nonzero_delta_lb}
\newcommand{\lemmalowerbounddeltanonzero}{
	Let $\mcM : \Sm \to \{-1,0,+1\}^{m}$ be an $(\varepsilon,\delta)$-locally
	differentially private mechanism. Let $\mathcal{T}\subseteq\{-1,0,+1\}^{m}$ denote the subset of
	vectors with exactly $r/2$ entries equal to $+1$ and $r/2$ entries equal to
	$-1$. Set $M' = \binom{m}{\lfloor m/2\rfloor}$ and $K\;:=\;|\mathcal{T}|\;=\;\binom{m}{r/2,\,r/2,\,m-r}$.
	% \[
	% M'\;:=\;|\mathcal{S}|
	% \;=\;
	% \begin{cases}
	% 	\binom{m}{m/2} & m\text{ even}\\
	% 	\binom{m}{\lfloor m/2\rfloor} & m\text{ odd}
	% \end{cases},
	% \qquad
	% K\;:=\;|\mathcal{T}|\;=\;\binom{m}{r/2,\,r/2,\,m-r}.
	% \]
	Let $\{S^{(\ell)}\}_{\ell=1}^{L}$ be independent random vectors drawn uniformly
	from $\Sm$. For each $\ell\in[L]$, let $\Shat^{(\ell)}=\mcM(S^{(\ell)})$,
	assume $\Shat^{(\ell)}\in\mathcal{T}$ almost surely, and define $V^{(\ell)}\;:=\;\sum_{j=1}^{m}\max\!\bigl\{0,\,\Shat^{(\ell)}_{j}S^{(\ell)}_{j}\bigr\}
	\;\in\;\{0,1,\dots,r\}$.
	For each $u\in\{0,1,\dots,r\}$, define the pointwise bound
	\[
	\pi_{\varepsilon}(u)\;:=\,\;
	\sum_{\substack{\alpha_1,\alpha_2\in\{0,\dots,r/2\}\\ \alpha_1+\alpha_2=u}}
	\binom{m-r}{\lceil(m-r)/2\rceil-(\alpha_1-\alpha_2)}
	\binom{r/2}{\alpha_1}\binom{r/2}{\alpha_2}\cdot\,
	\frac{e^{\varepsilon}\;+\;\tfrac{M'-1}{M'}\,K\,\delta}{e^{\varepsilon}+M'-1}.
	\]
	Define the upper tail $P_{\varepsilon}(v):=\sum_{u= v}^{r}\pi_{\varepsilon}(u)$.
	Then for every $v\in\mathbb{R}$,
	\begin{equation}\label{eq:lb-mean_delta_nonzero}
		\Pr\!\left[\frac{1}{L}\sum_{\ell=1}^{L}V^{(\ell)}\ge v\right]
		\;\le\;
		\inf_{\lambda\ge 0}\exp\!\left(
		L\log\!\Big(\sum_{u=0}^{r}e^{\lambda u}\pi_{\varepsilon}(u)\Big)-\lambda L v
		\right)
	\end{equation}
	\begin{equation}\label{eq:lb-median_delta_nonzero}
		\Pr\!\left(\mathrm{Median}\!\left(\{V^{(\ell)}\}_{\ell=1}^{L}\right)\ge v\right)
		\;\le\;
		\binom{L}{\lfloor L/2\rfloor}\,P_{\varepsilon}(v)^{\lceil L/2\rceil}.
	\end{equation}
}

We now state a version of this auditing lemma \cref{lemma:lower_bound_overlap_score} under the case of $\delta>0$.

%%\cref{lemma:lower_bound_overlap_score}
\begin{lemma}[Auditing test with $\delta>0$]{\label{lemma:lemma:lower_bound_overlap_score_delta_nonzero}}
		\lemmalowerbounddeltanonzero
\end{lemma}

Additionally, for the test to be meaningful, we need to ensure that the upper bound is monotonically non-decreasing in $\varepsilon$ (discussed in \cref{sec:redn_local_dp}) which holds when $\delta<M'/K$. A discussion on the same is done below.

\begin{remark}[Practical implications for $\delta > 0$]
For the audit's hypothesis test to remain meaningful when $\delta > 0$, two distinct practical issues arise.

\paragraph{Vacuity from the $\delta$-slack.} The pointwise bound $\pi_\varepsilon(u)$ contains an additive contribution proportional to $\frac{(M' - 1)\,K\,\delta}{M'}$. Since $M' \ge 1$, this factor is essentially $K \delta$, and $K = \binom{m}{r/2,\, r/2,\, m-r}$ grows rapidly with $m$ and $r$: for $m = 4500$ and $r \in \{2000, 3000\}$, Stirling gives $\ln K \approx 4475$, so even $\delta = 10^{-10}$ leaves $K \delta$ orders of magnitude larger than any realistic $\varepsilon_{\mathrm{LB}}$. The audit becomes vacuous in this regime regardless of the monotonicity condition below.

\paragraph{Monotonicity for a valid hypothesis test.} For the test $H_0 : \varepsilon \le \varepsilon_{\mathrm{LB}}$ to control Type-I error uniformly over the null, the upper bound must be monotone non-decreasing in $\varepsilon$ -- otherwise rejection at $\varepsilon = \varepsilon_{\mathrm{LB}}$ does not imply rejection for all $\varepsilon \le \varepsilon_{\mathrm{LB}}$. By \cref{cor:bound_monotonicity_delta_nonzero} this requires $\delta < M'/K$. Asymptotically, $M' \ge K$ iff $H_2(\alpha) + \alpha \le 1$ for $\alpha = r/m$, giving a threshold $\alpha^\ast \approx 0.227$. For $m = 4500$ this means any $r \lesssim 1020$ gives $M' \ge K$, in which case $\delta < 1$ trivially suffices; e.g.\ $r = 1000$ is comfortably inside this regime. The audit is therefore well-posed across the operating points used in our experiments.
\end{remark}
%%%Also, the monotonicity of the upper bound only 

\begin{proof}[Proof of \cref{lemma:lemma:lower_bound_overlap_score_delta_nonzero}]
	We first establish the pointwise bound
	\begin{equation}\label{eq:pw_delta_nonzero}
		\Pr\bigl[V^{(\ell)}=u\bigr]\;\le\;\pi_{\varepsilon}(u)
		\qquad\text{for every }u\in\{0,1,\dots,r\}.
	\end{equation}
	Fix $\ell$ and write $S=S^{(\ell)}$, $\Shat=\Shat^{(\ell)}$, $V=V^{(\ell)}$.
	
	\paragraph{Marginal joint bound from LDP.}
	Fix $s\in\mathcal{S}$ and $t\in\mathcal{T}$. By $(\varepsilon,\delta)$-LDP,
	$\Pr[\Shat=t\mid S=s]\le e^{\varepsilon}\Pr[\Shat=t\mid S=s']+\delta$ for every
	$s'\in\mathcal{S}$. Averaging this inequality over $s'\in\mathcal{S}\setminus\{s\}$
	and using
	$\sum_{s'\in\mathcal{S}}\Pr[\Shat=t\mid S=s']=M'\Pr[\Shat=t]$ (uniform prior on
	$\mathcal{S}$),
	\[
	\Pr[\Shat=t\mid S=s]
	\Bigl(1+\tfrac{e^{\varepsilon}}{M'-1}\Bigr)
	\;\le\;
	\tfrac{e^{\varepsilon}M'}{M'-1}\Pr[\Shat=t]+\delta,
	\]
	which rearranges to
	\[
	\Pr[\Shat=t\mid S=s]
	\;\le\;
	\frac{e^{\varepsilon}M'\Pr[\Shat=t]+(M'-1)\delta}{e^{\varepsilon}+M'-1}.
	\]
	Multiplying by the prior $1/M'$ gives the joint bound
	\begin{equation}\label{eq:joint_delta_nonzero}
		\Pr[S=s,\,\Shat=t]
		\;\le\;
		\frac{e^{\varepsilon}\Pr[\Shat=t]+\tfrac{M'-1}{M'}\delta}{e^{\varepsilon}+M'-1}.
	\end{equation}
	
	\paragraph{Counting argument.}
	For $(s,t)\in\Sm\times\mathcal{T}$, let
	$\alpha_1:=|\{j:t_j=+1,s_j=+1\}|$ and $\alpha_2:=|\{j:t_j=-1,s_j=-1\}|$,
	both in $\{0,\dots,r/2\}$. Then
	$V(s,t)=\sum_j\max\{0,s_jt_j\}=\alpha_1+\alpha_2$.
	For each fixed $t\in\mathcal{T}$, the number of $s\in\Sm$ with
	$V(s,t)=u$ is exactly
	\[
	N_u
	\;=\;
	\sum_{\substack{\alpha_1,\alpha_2\in\{0,\dots,r/2\}\\ \alpha_1+\alpha_2=u}}
	\binom{m-r}{\lceil(m-r)/2\rceil-(\alpha_1-\alpha_2)}
	\binom{r/2}{\alpha_1}\binom{r/2}{\alpha_2}.
	\]
	Indeed: pick the $\alpha_1$ agreement positions among the $r/2$
	coordinates where $t_j=+1$ ($\binom{r/2}{\alpha_1}$ ways) and the
	$\alpha_2$ agreement positions among the $r/2$ coordinates where
	$t_j=-1$ ($\binom{r/2}{\alpha_2}$ ways). The total $+1$/$-1$
	imbalance of $s$ is then $2(\alpha_1-\alpha_2)+2\beta-(m-r)$, where
	$\beta$ is the number of $+1$ entries of $s$ on the zero block of
	$t$. The constraint $s\in\Sm$ requires this imbalance to lie in
	$\{0,1\}$, which forces
	$\beta=\lceil(m-r)/2\rceil-(\alpha_1-\alpha_2)$
	(and is the unique integer solution in either parity of $m-r$),
	giving the $\binom{m-r}{\lceil(m-r)/2\rceil-(\alpha_1-\alpha_2)}$
	factor. By the symmetry of $V$ under coordinate permutations,
	$N_u(t)=N_u$ does not depend on $t$, hence
	$|\{(s,t):V(s,t)=u\}|=K\,N_u$.
	
	\paragraph{Assembling \eqref{eq:pw_delta_nonzero}.}
	Summing \eqref{eq:joint_delta_nonzero} over the pairs $(s,t)\in\mathcal{S}\times\mathcal{T}$
	with $V(s,t)=u$,
	\[
	\Pr[V=u]
	\;\le\;
	\frac{e^{\varepsilon}\sum_{(s,t):V=u}\Pr[\Shat=t]
		\;+\;\tfrac{M'-1}{M'}\delta\cdot\bigl|\{(s,t):V=u\}\bigr|}{e^{\varepsilon}+M'-1}.
	\]
	The first sum equals $\sum_{t\in\mathcal{T}}\Pr[\Shat=t]\,N_u=N_u$ (using
	$N_u(t)=N_u$ and $\sum_{t\in\mathcal{T}}\Pr[\Shat=t]=1$); the second count
	equals $KN_u$. Hence
	\[
	\Pr[V=u]\;\le\;
	N_u\cdot\frac{e^{\varepsilon}+\tfrac{M'-1}{M'}K\delta}{e^{\varepsilon}+M'-1}
	\;=\;\pi_{\varepsilon}(u),
	\]
	proving \eqref{eq:pw_delta_nonzero}. Summing over $u\ge v$ gives the marginal upper tail
	\begin{equation}\label{eq:tail_delta_nonzero}
		\Pr[V\ge v]\;\le\;P_{\varepsilon}(v).
	\end{equation}
	
	\paragraph{Median bound \eqref{eq:lb-median_delta_nonzero}.}
	The event
	$\mathrm{Median}(\{V^{(\ell)}\}_{\ell=1}^{L})\ge v$ implies the existence of
	some $J\subseteq[L]$ with $|J|=\lceil L/2\rceil$ such that
	$V^{(\ell)}\ge v$ for every $\ell\in J$. By independence and a union bound over
	the $\binom{L}{\lceil L/2\rceil}$ such subsets, together with \eqref{eq:tail_delta_nonzero},
	\[
	\Pr\!\left(\mathrm{Median}\!\left(\{V^{(\ell)}\}_{\ell=1}^{L}\right)\ge v\right)
	\;\le\;\binom{L}{\lceil L/2\rceil}P_{\varepsilon}(v)^{\lceil L/2\rceil}.
	\]
	
	\paragraph{Chernoff bound \eqref{eq:lb-mean_delta_nonzero}.}
	Fix $\lambda\ge 0$. By \eqref{eq:pw},
	\[
	\E\!\left[e^{\lambda V^{(\ell)}}\right]
	\;=\;\sum_{u=0}^{r}e^{\lambda u}\Pr[V^{(\ell)}=u]
	\;\le\;\sum_{u=0}^{r}e^{\lambda u}\pi_{\varepsilon}(u).
	\]
	By independence across $\ell$ and Markov's inequality applied to
	$e^{\lambda\sum_{\ell}V^{(\ell)}}$,
	\[
	\Pr\!\left[\frac{1}{L}\sum_{\ell=1}^{L}V^{(\ell)}\ge v\right]
	\;\le\;e^{-\lambda Lv}
	\Bigg(\sum_{u=0}^{r}e^{\lambda u}\pi_{\varepsilon}(u)\Bigg)^{\!L}
	\;=\;\exp\!\left(L\log\!\Big(\sum_{u=0}^{r}e^{\lambda u}\pi_{\varepsilon}(u)\Big)-\lambda Lv\right).
	\]
	Optimizing over $\lambda\ge 0$ yields \eqref{eq:lb-mean}.
\end{proof}

We now state a corollary showing that the upper bounds for median and mean statistic proven above is monotonic in $\varepsilon$ and thus, to reject the hypothesis $\varepsilon \leq \varepsilon_{\text{LB}}$, it is sufficient to evaluate the bound at $\varepsilon= \varepsilon_{\text{LB}}$ for $\delta< \frac{M'}{K}$.

\begin{corollary}[Monotonicity of audit bounds]\label{cor:bound_monotonicity_delta_nonzero}
Fix any $\delta \in [0, M'/K)$. Then the right-hand sides of \cref{eq:lb-mean} and \cref{eq:lb-median} in \cref{lemma:lower_bound_overlap_score} are monotone non-decreasing in $\varepsilon$.
\end{corollary}

\begin{proof}
Fix $\delta \in [0, M'/K)$. The pointwise bound factorises as $\pi_\varepsilon(u) = C(u) \cdot g(\varepsilon)$, where
\[
C(u) := \sum_{\substack{\alpha_1, \alpha_2 \in \{0, \dots, r/2\} \\ \alpha_1 + \alpha_2 = u}}
\binom{m-r}{\lceil(m-r)/2\rceil - (\alpha_1 - \alpha_2)}
\binom{r/2}{\alpha_1}\binom{r/2}{\alpha_2} \ge 0
\]
is independent of $\varepsilon$, and
\[
g(\varepsilon) := \frac{e^\varepsilon + a}{e^\varepsilon + b},
\qquad
a := \frac{M' - 1}{M'} K \delta, \quad b := M' - 1.
\]
Differentiating gives $g'(\varepsilon) = \dfrac{e^\varepsilon(b - a)}{(e^\varepsilon + b)^2}$. The condition $\delta < M'/K$ is equivalent to
\[
\frac{M' - 1}{M'} K \delta \;<\; M' - 1,
\quad \text{i.e.,} \quad a < b,
\]
so $g'(\varepsilon) > 0$ for all $\varepsilon \ge 0$. Hence $g$ is strictly increasing, and $\pi_\varepsilon(u) = C(u) \cdot g(\varepsilon)$ is monotone non-decreasing in $\varepsilon$ for every fixed $u \in \{0, 1, \dots, r\}$.

\textbf{Median bound (\cref{eq:lb-median_delta_nonzero}).} The upper tail $P_\varepsilon(v) = \sum_{u = \lceil v \rceil}^{r} \pi_\varepsilon(u)$ is a non-negative sum of monotone non-decreasing functions, and is therefore itself monotone non-decreasing in $\varepsilon$. Raising to the power $\lceil L/2 \rceil$ and multiplying by the positive constant $\binom{L}{\lfloor L/2 \rfloor}$ both preserve monotonicity.

\textbf{Mean bound (\cref{eq:lb-mean_delta_nonzero}).} Fix any $\lambda \ge 0$. The moment generating function $\sum_{u=0}^{r} e^{\lambda u} \pi_\varepsilon(u)$ is a non-negative linear combination of $\{\pi_\varepsilon(u)\}_u$, hence monotone non-decreasing in $\varepsilon$. Composing with $\log$ (monotone), multiplying by $L > 0$, subtracting the $\varepsilon$-independent term $\lambda L v$, and exponentiating all preserve monotonicity. Therefore, for each fixed $\lambda \ge 0$, the function inside the infimum is monotone non-decreasing in $\varepsilon$. A pointwise infimum of monotone non-decreasing functions is itself monotone non-decreasing: for $\varepsilon_1 \le \varepsilon_2$ and any $\lambda$, $f_\lambda(\varepsilon_1) \le f_\lambda(\varepsilon_2)$, so taking the infimum over $\lambda$ on both sides yields $\inf_\lambda f_\lambda(\varepsilon_1) \le \inf_\lambda f_\lambda(\varepsilon_2)$. This completes the proof.
\end{proof}

\section{Datasets, models, and auditor batching}\label{sec:appendix_datasets}

This appendix gives the full data-construction and model details for the two datasets described in \cref{sec:datasets}.

\subsection{CIFAR-100}\label{sec:appendix_cifar100_dataset}

\paragraph{Forget/retain construction.} For CIFAR-100 \citep{Krizhevsky09learningmultiple}, we first split the original training set into a training portion (90\%) and a validation portion (10\%). From the training portion, we designate $10\%$ of the points as the forget set $\mcD_f$ ($4{,}500$ points), and use the remaining $40{,}500$ points as the retain set $\mcD_r$. We consider two types of splits. In the \emph{uniform} split, the forget and retain sets are formed by sampling points uniformly at random. In the \emph{adversarial} split, the forget and retain sets are chosen from largely disjoint CIFAR-100 label classes, with overlap allowed in at most one label class. We then independently shuffle the retain and forget sets, and partition both $\mcD_r$ and $\mcD_f$ into batches of size $B$.

\paragraph{Auditor batching.} For uncertified unlearning algorithms we use $B=1$, while for certified unlearning algorithms we use $B=750$. Accordingly, we use the \emph{batchwise inclusion/exclusion} auditor (\cref{alg:instantiation_I}) for uncertified methods and the \emph{joint sign-vector} auditor (\cref{alg:instantiation_II}) for certified methods. A discussion of why this split between auditor instantiations is appropriate is given in \cref{sec:instantiation_discussion_certified_uncertified}, with a more detailed study in \cref{sec:appendix_batch_size_varying_study}.

\paragraph{Model.} We use a \texttt{TinyNetCIFAR100} model, a lightweight CNN that attains around 55\% accuracy on the CIFAR-100 test set after $400$ epochs of training. The architecture consists of two convolutional blocks, each containing two $3\times3$ convolutional layers (with 64 and 128 filters respectively) with Group Normalization (8 groups) and ReLU activations, followed by $2\times2$ max pooling. After the two blocks, global average pooling collapses the spatial dimensions, and a single linear layer maps to the 100-class output. All weights are Kaiming-initialized. The model contains no dropout or batch normalization; Group Normalization is used throughout for training stability. We use a cosine or one-cycle learning-rate scheduler with stochastic gradient descent optimizer.

%%\sss{Need a description of the TinyNetCIFAR100 architecture too.}

\subsection{Shakespeare}\label{sec:appendix_shakespeare_dataset}

\paragraph{Forget/retain construction.} Following \citet{mcmahan2017communication}, we treat each speaking role in each play as a separate client, parsing speaker labels from the text to extract per-role dialogue. Roles with fewer than $2$ dialogue blocks are discarded, and for our experiments we randomly subsample $300$ roles. Each role's dialogue blocks are split chronologically into train (80\%), validation (10\%), and test (10\%); validation and test sets are formed by aggregating across all $300$ roles. The training data is then partitioned into a retain set $\mcD_r$ and a forget set $\mcD_f$ \emph{at the role level} --- each role is assigned entirely to one set, ensuring zero character overlap between $\mcD_r$ and $\mcD_f$. Roles are randomly shuffled and greedily assigned to $\mcD_f$ until approximately $10\%$ of training characters are accumulated, yielding $33$ forget roles ($\mcD_f$, $135{,}300$ characters, ${\sim}10.0\%$ of train) and $267$ retain roles ($\mcD_r$, $1{,}221{,}363$ characters, ${\sim}90.0\%$ of train).

For unlearning experiments, the forget set $\mcD_f$ is further partitioned into $B=400$ equal batches: the dialogue blocks from all forget roles are shuffled at the block level (preserving within-block coherence) and split into $400$ contiguous chunks $\mcD_{f,1}, \ldots, \mcD_{f,400}$. For each run, $200$ forget chunks are sampled uniformly at random from the $400$. Each chunk and the retain set are independently converted to sliding-window samples of length $80$ characters, where $x$ is the input sequence and $y$ is the target sequence shifted by one character. The retain samples and sampled forget samples are then pooled and shuffled to form the training dataset. The retain and forget datasets used during evaluation are constructed by the same sliding-window procedure applied to the retain and sampled forget text respectively.

\paragraph{Auditor batching.} On Shakespeare we evaluate only uncertified unlearning methods (gradient ascent, interleaved ascent--descent, fine-tuning on retain, and Hessian-based unlearning of \citet{zhang2025certifiedunlearningdeepneural}); we therefore use the \emph{batchwise inclusion/exclusion} auditor (\cref{alg:instantiation_I}) throughout, as on CIFAR-100.

\paragraph{Model.} The model is a $2$-layer stacked character-level LSTM following \citet{mcmahan2017communication}. Each input character is embedded into an $8$-dimensional space, processed through two LSTM layers each with $256$ hidden units, and projected to a softmax output layer over the vocabulary of $|\mcV|=100$ characters. The full model has $825{,}220$ parameters and is trained with an unroll length of $L=80$ characters. For the Shakespeare dataset, the score $\phi(f(x)_y)$ for a sample $(x,y)$ is the mean per-character cross-entropy loss $-\tfrac{1}{T}\sum_{t=1}^{T}\log p_\theta(y_t \mid x_{\le t})$ assigned by the model $f$, where $T=80$ is the sequence length.

\section{Robustness of the auditor: alternative score and aggregation}

In this appendix we sanity-check two design choices of our auditor. First, we replace the logit score $\phi(f(x)_y)$ used throughout the paper with the per-example cross-entropy loss and recompute the lower bounds (\cref{sec:appendix_logit_cross_entropy_comparison}). Second, we replace the mean aggregation of overlap scores with the median (\cref{sec:mean_median_comparison}). In both cases the directional trends with respect to the support size $r$ are preserved, indicating that our reported bounds are not artefacts of a particular score function or aggregation rule.

\subsection{Logit score versus cross-entropy score}\label{sec:appendix_logit_cross_entropy_comparison}

\begin{table}[t]
	\centering
	\small
	\setlength{\tabcolsep}{5pt}
	\begin{tabular}{lcccccccccc}
		\toprule
		\diagbox{Method}{$r$}
		& \multicolumn{2}{c}{200}
		& \multicolumn{2}{c}{1000}
		& \multicolumn{2}{c}{2000}
		& \multicolumn{2}{c}{3000}
		& \multicolumn{2}{c}{4000} \\
		\cmidrule(lr){2-3} \cmidrule(lr){4-5} \cmidrule(lr){6-7} \cmidrule(lr){8-9} \cmidrule(lr){10-11}
		& Logit & Loss & Logit & Loss & Logit & Loss & Logit & Loss & Logit & Loss \\
		\midrule
		uni.\ (pure ascent)      & 29.5805 & 18.503  & 80.693  & 62.7975 & 90.1105 & 76.913  & 79.835  & 69.6605 & 66.631  & 61.5315 \\
		adv.\ (pure ascent)      & 25.359  & 23.89   & 76.2265 & 75.746  & 91.813  & 90.557  & 86.2075 & 86.5575 & 71.913  & 72.077  \\
		uni.\ (IDA, $q=1$)       & 34.298  & 17.965  & 98.717  & 79.331  & 111.161 & 98.892  & 100.6055 & 92.2955 & 82.4575 & 79.0315 \\
		adv.\ (IDA, $q=1$)       & 0.1545  & 0.0885  & 0.4075  & 1.22    & 0.1195  & 1.4075  & 0.0      & 1.442   & 0.0      & 1.1425 \\
		uni.\ (IDA, $q=2$)       & 31.6765 & 18.776  & 95.925  & 78.1045 & 111.8135 & 96.469  & 100.2645 & 90.7795 & 82.575  & 78.9165 \\
		adv.\ (IDA, $q=2$)       & 7.659   & 6.7045  & 15.7665 & 15.4605 & 15.8635 & 16.0625 & 15.064   & 15.482  & 12.9405 & 12.9865 \\
		uni.\ (IDA, $q=4$)       & 28.122  & 17.502  & 91.4265 & 76.8295 & 106.118 & 91.0045 & 95.5885 & 87.33   & 80.271  & 72.159  \\
		adv.\ (IDA, $q=4$)       & 17.7655 & 17.2405 & 60.144  & 59.1525 & 69.435  & 69.087  & 66.1155 & 66.207  & 56.3515 & 56.9075 \\
		\bottomrule
	\end{tabular}
	\caption{Lower bounds from logit score versus cross-entropy loss, across pure ascent and interleaved ascent--descent (IDA) at varying $q$, for both uniform (uni.) and adversarial (adv.) splits.}
	\label{tab:logit_loss_comparison_large_k}
\end{table}

\Cref{tab:logit_loss_comparison_large_k} compares the lower bounds obtained when the auditor's per-example score is the cross-entropy loss versus the logit score, for pure ascent and interleaved ascent--descent (IDA) at $q\in\{1,2,4\}$ on both the uniform and adversarial splits. The two scores yield similar directional trends in $r$, with the logit score giving slightly tighter bounds in most settings.

\subsection{Mean versus median aggregation}\label{sec:mean_median_comparison}

\begin{table}[t]
	\centering
	\small
	\setlength{\tabcolsep}{5pt}
	\begin{tabular}{lcccccccccc}
		\toprule
		\diagbox{Method}{$r$}
		& \multicolumn{2}{c}{200}
		& \multicolumn{2}{c}{1000}
		& \multicolumn{2}{c}{2000}
		& \multicolumn{2}{c}{3000}
		& \multicolumn{2}{c}{4000} \\
		\cmidrule(lr){2-3} \cmidrule(lr){4-5} \cmidrule(lr){6-7} \cmidrule(lr){8-9} \cmidrule(lr){10-11}
		& Mean & Median & Mean & Median & Mean & Median & Mean & Median & Mean & Median \\
		\midrule
		uni.\ (pure ascent)      & 29.5805 & 30.043  & 80.693  & 77.6435 & 90.1105 & 92.2745 & 79.835  & 80.6945 & 66.631  & 67.9155 \\
		adv.\ (pure ascent)      & 25.359  & 25.6695 & 76.2265 & 79.4695 & 91.813  & 95.4495 & 86.2075 & 90.0365 & 71.913  & 75.5195 \\
		uni.\ (IDA, $q=1$)       & 34.298  & 33.917  & 98.717  & 101.2995 & 111.161 & 111.233 & 100.6055 & 99.551  & 82.4575 & 81.4985 \\
		adv.\ (IDA, $q=1$)       & 0.1545  & 0.115   & 0.4075  & 0.3685  & 0.1195  & 0.1425  & 0.0     & 0.0  & 0.0     & 0.0 \\
		uni.\ (IDA, $q=2$)       & 31.6765 & 30.9785 & 95.925  & 96.3925 & 111.8135 & 113.2385 & 100.2645 & 101.8195 & 82.575  & 82.957  \\
		adv.\ (IDA, $q=2$)       & 7.659   & 7.2605  & 15.7665 & 16.5665 & 15.8635 & 17.5945 & 15.064  & 16.7695 & 12.9405 & 15.2445 \\
		uni.\ (IDA, $q=4$)       & 28.122  & 30.043  & 91.4265 & 90.307  & 106.118 & 105.3365 & 95.5885 & 96.569  & 80.271  & 80.053  \\
		adv.\ (IDA, $q=4$)       & 17.7655 & 21.0215 & 60.144  & 62.3215 & 69.435  & 76.0615 & 66.1155 & 73.8025 & 56.3515 & 63.249  \\
		\bottomrule
	\end{tabular}
	\caption{Lower bounds from mean versus median aggregation of overlap scores, across pure ascent and interleaved ascent--descent (IDA) at varying $q$, for both uniform (uni.) and adversarial (adv.) splits.}
	\label{tab:mean_median_comparison_large_k}
\end{table}
\Cref{tab:mean_median_comparison_large_k} compares the lower bounds obtained from the mean and median aggregations of the overlap scores. The two aggregations agree to within a few percent in nearly all settings, and both reproduce the inverted-U dependence on $r$ noted in \cref{sec:appendix_batch_size_varying_study} --- the lower bound first increases with $r$ and then decreases.

\section{Varying batch sizes $B$ (and equivalently $m$) and reported support size $r$}\label{sec:appendix_batch_size_varying_study}

This appendix gives the detailed analysis of how the lower bound $\varepsilon_{\mathrm{LB}}$ depends on the number of forget batches $m$ (equivalently the audit batch size $B$, since $m=|\mcD_f|/B$) and the reported support size $r$. The compact takeaways are summarised in \cref{remark:auditor_design_choices}.

\subsection{Maximum attainable bound under perfect prediction (in $m$)}

The maximum lower bound on $\varepsilon$ implied by a perfect prediction --- overlap score $m$ with support size $r=m$ --- grows with $m$. For instance, applying \cref{lemma:epsilon_lower_bound} with $\Gamma=10$ independent runs yields $\varepsilon_{\mathrm{LB}} \ge 4.00$ at $m=6$, but $\varepsilon_{\mathrm{LB}} \ge 67.87$ at $m=100$. Intuitively, identifying the correct sign vector is harder when $m$ is large, since the number of candidates $|\mcS_m|$ grows exponentially in $m$. A perfect prediction at large $m$ therefore reflects much stronger distinguishing power, and certifies a correspondingly larger lower bound on $\varepsilon$.

\subsection{Choice of audit batch size $B$}

Since $m=|\mcD_f|/B$, smaller batches enlarge the hypothesis space $\mcS_m$ and raise the attainable lower bound --- but only if the adversary can pick the right hypothesis. Uncertified algorithms admit a strong per-batch signal, so small $B$ (large $m$) is preferable. Certified algorithms have tight guarantees and weak per-batch signal; here, large $m$ leaves the adversary unable to beat random guessing among exponentially many candidates, and one must use a larger $B$ (smaller $m$) to obtain any nonzero bound. In addition, as discussed in \cref{sec:auditor_instantiations}, we use the \emph{joint sign vector} prediction for certified unlearning algorithms, which is precisely why having a small batch size (large $m$) is computationally infeasible. Thus, in \cref{sec:unlearning_algorithms} we stick to $B=1$ for uncertified algorithms and $B=750$ for certified algorithms on CIFAR-100.

\subsection{Dependence on the support size $r$}

The support size $r$ lets the auditor abstain on uncertain batches, improving the overlap on those it does report. Given $m$, the maximum lower bound \emph{conditional on perfect overlap} grows with $r$: at $m=100$, perfect overlap yields $\varepsilon_{\mathrm{LB}} \ge 40.83$ at $r=60$, rising to $\varepsilon_{\mathrm{LB}} \ge 67.87$ at $r=100$. In practice, the empirical lower bound traces an inverted-U in $r$. It first increases for two compounding reasons: the maximum attainable bound at perfect overlap grows with $r$, and the auditor's reported predictions remain accurate while there are still confident batches to report. As $r$ grows further, however, perfect overlap becomes harder to attain --- the auditor is forced to commit to low-confidence batches, which dilute the overlap score and pull the bound back down. 

In \cref{tab:batch_size_comparison}, we plot the lower bound against $r \cdot B$ (equivalently, $4500\,r/m$, since $m = |\mcD_f|/B = 4500/B$) for varying batch sizes $B \in \{1, 10, 100, 500\}$ on the ascent-on-forget unlearning algorithm. For both $B=1$ and $B=10$, the lower bound first increases with $r$ and then decreases. This inverted-U effect is more pronounced at smaller batch sizes (larger $m$).

\begin{table}[t]
	\centering
	\small
	\setlength{\tabcolsep}{7pt}
	\begin{tabular}{lcccc}
		\toprule
		\diagbox{$r \cdot B$}{$B$}
		& $B=1$ ($m=4500$)
		& $B=10$ ($m=450$)
        & $B=100$ ($m=45$)
		& $B=500$ ($m=9$) \\
		\midrule
		100   & 17.4945  & 3.31 & N/A  & N/A   \\
		200   & 29.5805  & 6.46  & 0.53 & N.A   \\
		400   & 47.8495  & 13.06 & 1.21 & N.A   \\
		1000  & 80.693   & 29.74 & 3.26 & 0.00     \\
		2000  & 90.1105  & 51.47 & 6.64  & 0.27 \\
		3000  & 79.835   & 64.08 & 9.99 & 0.76 \\
		3500  & 73.21    & 65.19  & N/A  & N/A   \\
		4000  & 66.631   & 64.22 & 13.26 & 1.29 \\
		\bottomrule
	\end{tabular}
	\caption{Comparison across batch sizes for the ascent on forget set setting.}
	\label{tab:batch_size_comparison}
\end{table}

% \begin{table}[t]
% 	\centering
% 	\small
% 	\setlength{\tabcolsep}{7pt}
% 	\begin{tabular}{lcccc}
% 		\toprule
% 		\diagbox{$k \times bs $}{Method}
% 		& $bs=1 (m = 4500)$
% 		& $bs=10 (m=450)$
%         & $bs=100(m = 45)$
% 		& $bs=500 (m=9) $ \\
% 		\midrule
% 		100   & 34.989  & 6.621 & N/A & N/A     \\
% 		200  & 59.161  & 12.912 &1.06 & N.A     \\
% 		400  & 95.699  & 26.122 &2.42 & N.A     \\
% 		1000  & 161.386 & 59.488 &6.51 & 0     \\
% 		2000 & 180.221 & 102.936 & 13.28 & 0.533 \\
% 		3000 & 159.670 & 128.155 & 19.98 & 1.514 \\
% 		3500 & 146.42	& 130.38 & N/A & N/A\\
% 		4000 & 133.262 & 128.445 & 26.52 & 2.584 \\
% 		\bottomrule
% 	\end{tabular}
% 	\caption{Comparison across batch sizes for the ascent on forget set setting.}
% 	\label{tab:batch_size_comparison}
% \end{table}

% \begin{corollary}[Monotonicity of audit bounds \cref{lemma:lower_bound_overlap_score}]\label{cor:bound_monotonicity}
% Fix any $\delta \in [0, M'/K)$. Then the right-hand sides of \cref{eq:lb-mean} and \cref{eq:lb-median} in \cref{lemma:lower_bound_overlap_score} are monotone non-decreasing in $\varepsilon$.
% \end{corollary}

\subsection{Batchwise inclusion/exclusion versus joint sign-vector prediction}\label{sec:appendix_instantiation_study}\label{sec:instantiation_discussion_certified_uncertified}

The choice of auditor instantiation depends on whether the unlearning algorithm is certified. For certified methods such as model clipping \citep{koloskova2025certifiedunlearningneuralnetworks} and R2D \citep{rewind2delete}, Gaussian noise injected during unlearning makes the in- and out-distributions for a single batch nearly indistinguishable, so per-batch likelihood ratios rarely yield overlap scores above the random baseline of $r/2$. We therefore use the \emph{joint sign-vector} predictor, which calibrates a separate Gaussian for every candidate $\widetilde S \in \Sm$ and aggregates evidence across all forget points before committing to a single sign vector. Since this requires independent calibration runs for each of the $|\Sm|$ candidates, we restrict to small $m$ --- concretely, batch size $B=750$ giving $m=6$ forget batches. For uncertified methods, the per-batch signal is strong enough that this issue does not arise; we obtain the strongest bounds at large $m$, where $|\Sm|$ is astronomically large and enumeration is infeasible, and so default to the \emph{batchwise inclusion/exclusion} predictor.

\paragraph{Experiment with batchwise inclusion/exclusion.} We also evaluated the \emph{batchwise inclusion/exclusion} auditor on the uncertified variant of model clipping, with \(C_2 = 5\) and \(\sigma = 10^{-4}\) as in \cref{sec:certified_unlearning_audited}. We used batch sizes \(B=750\) and \(B=100\), with \(r=6\) and \(r=40\), respectively.  The resulting lower bounds were only \(0.112\) and \(0.007\) even just after the \emph{first} step of unlearning.  We believe this is because, after noise is added to the model at each step, the predictions of the resulting models under inclusion and exclusion of a forget batch start to overlap substantially. Once this overlap becomes large, distinguishing the two cases through batchwise membership inference becomes difficult, leading to only a negligible lower bound.

%We possibly think a reason for this might be the fact that after adding noise to the model post every step, the model gets too noisy and doing a membership inference attack to predict whether a batch was included or excluded in the forget set becomes too hard to predict.
%
%We believe this is because clipping combined with post-update Gaussian noise significantly reduces the signal-to-noise ratio of the batchwise membership signal. As a result, the model distributions induced by including versus excluding a particular forget batch exhibit substantial overlap, making the associated likelihood-ratio test only marginally better than random guessing. In effect, the Gaussian perturbation masks the marginal contribution of an individual forget batch, so a batchwise membership inference attack becomes too weak to certify a nontrivial lower bound.
\section{Unlearning algorithms: experimental details}\label{sec:appendix_unlearning_algorithms}

This appendix gives the full setup for each unlearning algorithm audited in \cref{sec:unlearning_algorithms}. \cref{sec:appendix_model_clipping_setup,sec:appendix_rewind_to_delete} cover the certified algorithms; \cref{sec:appendix_uncertified_methods,sec:appendix_hessian_unlearning,sec:appendix_ida_retraining,sec:appendix_ascent_on_forget,sec:appendix_pure_finetune} cover the four heuristic algorithms together with their base-training configurations.

\subsection{Model clipping}\label{sec:appendix_model_clipping_setup}

\subsubsection{Algorithm}

Let $\hat{x}$ denote the trained model produced by $\mcA$ prior to unlearning. The unlearning procedure of \citet{koloskova2025certifiedunlearningneuralnetworks} consists of two phases.

\paragraph{Phase 1: noisy projection.} For $t = 0, 1, \dots, T_{\mathrm{noisy}} - 1$, the iterate is updated by
\begin{align}
	x_0 = \hat{x} + \xi_0, \qquad
	x_{t+1} = \Pi_{C_2}\!\left(x_t - \gamma\big(g_t + \lambda x_t\big)\right) + \xi_{t+1},
\end{align}
where $\Pi_{C_2}$ is projection onto the $\ell_2$-ball of radius $C_2$, $g_t$ is the gradient of the loss on the retain set evaluated at $x_t$, $\gamma$ is the step size, $\lambda$ is the $\ell_2$ regularisation strength, and the noise satisfies $\xi_0 \sim \mcN(0,\sigma_0^2 \sI_d)$ and $\xi_{t+1} \sim \mcN(0,\sigma^2 \sI_d)$ for $t \geq 0$. The number of noisy steps $T_{\mathrm{noisy}}$ is set by Theorem~4.2 of \citet{koloskova2025certifiedunlearningneuralnetworks} so that the iterate $x_{T_{\mathrm{noisy}}}$ satisfies the target $(\varepsilon,\delta)$ guarantee.

\paragraph{Phase 2: fine-tuning.} Starting from $x_{T_{\mathrm{noisy}}}$, the algorithm runs plain gradient descent on the retain-set loss --- with no noise and no projection --- until the total step budget is exhausted.

\subsubsection{Hyperparameters}

As chosen in \citet{koloskova2025certifiedunlearningneuralnetworks}, we train for a total of $400$ epochs with a training batch size of $128$ ($317$ optimisation steps per epoch), and unlearning is applied for a total of $50$ epochs. We apply a one-cycle learning schedule with maximum learning rate $10^{-3}$ and regulariser $5\times 10^{-4}$. For unlearning Phase~1, we use a constant learning rate of $10^{-3}$ for the noisy steps; for Phase~2 we use a one-cycle schedule with maximum $0.1$ and regulariser $5\times 10^{-4}$. The initial clipping radius just before unlearning starts is $C_0=30$. The same regulariser is applied during the retain-set fine-tuning phase. The sweep over the certified $(\varepsilon, C_2, \sigma)$ grid is reported in \cref{tab:epsilon_lb_model_clipping}.

\begin{table}[t]
	\centering
	\small
	\begin{tabular}{cccccc}
		\toprule
		$C$ & $\sigma$ & $\varepsilon$ & \begin{tabular}[c]{@{}c@{}}Noisy\\ steps\end{tabular} & \multicolumn{2}{c}{$\varepsilon_{\mathrm{LB}}$ (uniform/adversarial)} \\
		\cmidrule(lr){5-6}
		&  &  &  & Step $1$ & Epoch $50$ \\
		\midrule
		$0.2$   & $0.1$    & $1$       & $40$   & 0    & 0 \\
		$0.1$   & $0.02$   & $10$      & N/A    & 0    & 0 \\
		$0.225$ & $0.001$  & $10^{5}$  & $4041$ & 0.0 & 0 \\
		\bottomrule
	\end{tabular}
	\caption{Lower bounds for model clipping under uniform/adversarial CIFAR-100 splits, evaluated at the first noisy step and after $50$ epochs.}
	\label{tab:epsilon_lb_model_clipping}
\end{table}

% \subsubsection{Reporting convention}

% Because the unlearning procedure is iterative, we report the empirical lower bound $\varepsilon_{\mathrm{LB}}$ as a function of the step index $t$, treating the algorithm as if it had terminated after the first $t$ steps. The bound shown at the prescribed stopping time of Theorem~4.2 of \citet{koloskova2025certifiedunlearningneuralnetworks} is the one that compares directly against the certified target $\varepsilon$.

\subsubsection{Uncertified variant of model clipping}\label{sec:appendix_uncertified_model_clipping}

For the uncertified experiment of \cref{fig:model_clipping} we set $C_2 = 5$ and $\sigma = 10^{-4}$, terminate the clipping/noise-addition phase after $10$ epochs ($3{,}170$ update steps), and then fine-tune as above. The ratio $C_2/\sigma$ is far larger than what Theorem~4.2 of \citet{koloskova2025certifiedunlearningneuralnetworks} admits, so no certified $(\varepsilon,\delta)$ guarantee applies; we still report $\varepsilon_{\mathrm{LB}}$ at every step $t$ as defined above. To obtain tighter bounds in this regime we use $L=500$ runs. For context, $1$ epoch corresponds to $317$ steps in this setup, so \cref{fig:model_clipping} effectively covers up to roughly $2$ epochs.

\subsection{Rewind-to-delete}\label{sec:appendix_rewind_to_delete}
We take $T = 40$ training epochs and $K = 5$ rewind epochs. Following \citet[Theorem~3.1]{rewind2delete}, the unlearned model $\theta''_K$ is obtained by loading the checkpoint $\theta'_{T-K}$ and then training for $K$ epochs on the retain set; final-step Gaussian noise of scale $\sigma$ is added to ensure $(\varepsilon, \delta)$-indistinguishability between the trained model $\theta'_T$ and $\theta''_K$. Since we consider $\varepsilon > 1$, the noise parameter $\sigma$ \footnote{To compute $\sigma$, one needs the smoothness constant $L$ and a uniform bound on the gradient $G$, both estimated following \cite{rewind2delete}.
} must satisfy
\[
\Phi\!\left(-\tfrac{\varepsilon \sigma}{\Delta} + \tfrac{\Delta}{2\sigma}\right) - e^{\varepsilon}\, \Phi\!\left(-\tfrac{\varepsilon \sigma}{\Delta} - \tfrac{\Delta}{2\sigma}\right) \le \delta,
\]
where $\Phi$ is the standard Gaussian CDF and $\Delta = \|\theta'_T - \theta''_K\|$ is the sensitivity. Both $T$ and $K$ are kept small because the resulting $\sigma$ grows exponentially in $T$ and $K$. Theoretical guarantees for R2D hold only under full-batch gradient descent, so we use full-batch updates throughout; on CIFAR-100 this yields a base model with poor accuracy of around $4\%$. In our auditing setup, we use a learning rate of $0.01$. We swept over $\{0.1, 0.01, 0.001\}$; the resulting $\varepsilon_{\mathrm{LB}}$ values were all small and did not differ meaningfully across the sweep, so we report results for $0.01$ throughout.

% \subsection{Rewind-to-delete}\label{sec:appendix_rewind_to_delete}

% In our setup we take $T = 40$ epochs and $K = 5$ epochs. Both are kept small because the noise $\sigma$ satisfying \citep[Theorem~3.1]{rewind2delete} grows exponentially with $T$ and $K$. Theoretical guarantees for rewind-to-delete only hold under full-batch gradient descent, so we use full-batch updates throughout; on CIFAR-100 this yields a base model with poor accuracy of around $4\%$.

%%\sss{TO-DO: report $\varepsilon_{\mathrm{LB}}$ values for rewind-to-delete here.}

\subsection{Base training configurations for uncertified methods}\label{sec:appendix_uncertified_methods}

\paragraph{Shakespeare.} The base model is trained for $15$ epochs using SGD with learning rate $0.1$, batch size $256$, and gradient clipping of $1.0$, with no weight decay or learning-rate scheduling. The resulting accuracy is roughly $0.53$, in line with \citep{mcmahan2017communication}. Unlearning runs for a total of $8$ epochs; the retain-set fine-tuning phase uses learning rate $0.1$.

\paragraph{CIFAR-100.} For Hessian-based unlearning, we use a cosine scheduler with learning rate $0.1$ following \citet{zhang2025certifiedunlearningdeepneural}, training for a total of $400$ epochs. For the remaining three uncertified methods (IDA, ascent on the forget set, pure fine-tuning), the training configuration matches \cref{sec:appendix_model_clipping_setup}; the retain-set fine-tuning phase reverts to a one-cycle schedule with maximum learning rate $0.1$.

\subsection{Hessian-based unlearning}\label{sec:appendix_hessian_unlearning}

\paragraph{Update rule.} Let $g_f$ denote the forget-set gradient and $H$ the (regularised) Hessian. The unlearning step is
\begin{equation}\label{eq:hessian_update}
\theta \;\mapsto\; \theta - H^{-1} g_f + \xi, \qquad \xi \sim \mcN(0, \sigma^2 \sI),
\end{equation}
where $\sigma$ is the final noise scale and the Gaussian perturbation is added to every parameter.

\paragraph{LiSSA approximation.} As in \citet{zhang2025certifiedunlearningdeepneural}, we approximate the inverse-Hessian--vector product $H^{-1} g$ via the LiSSA algorithm of \citet{agarwallisasecondorder}, avoiding explicit construction of the Hessian. The LiSSA approximation is governed by four hyperparameters --- depth parameters $s_1$ and $s_2$, the \texttt{scale}, and the retain batch size --- together with a weight decay applied during the Hessian computation. We describe each below before giving the per-dataset values.

\paragraph{Depth parameters $s_1$ and $s_2$.} The parameter $s_1$ controls the number of independent LiSSA estimation runs whose results are averaged to reduce variance, while $s_2$ controls the depth of the Neumann-series recursion within each run, governing convergence to the true Newton direction.

\paragraph{Scale.} The \texttt{scale} parameter must strictly upper-bound the largest eigenvalue of the regularised Hessian to ensure that the Neumann series converges. Too small a value causes divergence, while too large a value slows convergence and undershoots the Newton step.

\paragraph{Unlearning batch size.} The unlearning batch size controls the size of the random retain mini-batches sampled at each recursion step, trading off gradient noise against computational cost per LiSSA iteration.

\paragraph{Weight decay.} The weight decay enters as a squared $L_2$ regularisation term added to the loss during the Hessian computation, effectively shifting the Hessian by $\lambda \mI$ and improving its conditioning for inversion. Crucially, it is applied only in the Hessian computation and not in the gradient $g$, so that the Newton update approximates $H^{-1} g$ where $H$ is the regularised Hessian but $g$ is the unregularised retain gradient.

\paragraph{Per-dataset values.} For both datasets we set the weight decay to $5 \times 10^{-4}$. For CIFAR-100 we use $s_1 = 10$, $s_2 = 1000$, $\texttt{scale} = 1000$, and an unlearning batch size of $10$. For Shakespeare we use $s_1 = 5$, $s_2 = 700$, $\texttt{scale} = 5000$, and an unlearning batch size of $256$ (matching the training batch size); $s_1$ and $s_2$ are chosen slightly smaller for Shakespeare for computational efficiency.

\paragraph{Final noise.} After the Hessian-based step, Gaussian noise of standard deviation $10^{-3}$ is added to every parameter, identically in both setups. We additionally perform an experiment below describing how varying the final noise changes the lower bound of unlearning parameter $\varepsilon$.

\paragraph{Audit by varying the final noise step.} We perform two additional experiments on CIFAR-100 (uniform splitting) keeping all parameters unchanged but only tweaking the final noise to $5\times 10^{-3}$ and $10^{-2}$ respectively in \cref{tab:noise_ablation_cifar100}.  Trends are along expected lines. Although, we see the computed lower bound gets smaller with noise, the test accuracy also reduces.

%%and we obtain a lower bound of 56.77 and 0.40 respectively. Recall from \cref{tab:lower_bound_comparison_cifar100} that the lower bound under final noise = 1e-3 (original setup) was 285. This is along expected lines and the monotonicty we see is expected that adding additional noise improves unlearning. However, the final test accura

\begin{table}[t]
	\centering
	\small
	\setlength{\tabcolsep}{10pt}
	\begin{tabular}{clc}
		\toprule
		Final noise \(\sigma\) & Lower bound & R/F/T \\
		\midrule
		\(10^{-3}\) (original)   & 142.5 &  0.69/0.67/0.51\\
		\(5 \times 10^{-3}\) & 28.5 &  0.53/0.51/0.42\\
		\(10^{-2}\)   & 0.20 & 0.22/0.21/0.19\\
		\bottomrule
	\end{tabular}
	\caption{$\varepsilon$ lower bound with noise. R/F/T denotes retain, forget and test accuracy}
	\label{tab:noise_ablation_cifar100}
\end{table}

\subsection{Interleaved descent--ascent (IDA)}\label{sec:appendix_ida_retraining}

\paragraph{Update rule.} Letting $g_t^{r}$ and $g_t^{f}$ denote the retain- and forget-loss gradients at step $t$, the IDA update for the first $E_f$ epochs is
\begin{equation}\label{eq:ida_update}
x_{t+1} = x_t - \eta_t \bigl(g_t^{r} - \lambda\, \mathbf{1}\{t \equiv 0 \!\!\pmod{q+1}\}\, g_t^{f}\bigr),
\end{equation}
so that the interleaved step (with the forget-gradient term active) fires once per $q+1$ updates and the remaining $q$ updates within each cycle reduce to plain retain descent $x_{t+1} = x_t - \eta_t g_t^{r}$. After $E_f$ epochs, only the latter (plain retain descent) is used.

\paragraph{CIFAR-100.} We fix $\lambda = 0.75$, vary $q \in \{1,2,4\}$, and set $E_f = 10$. The learning rate is held constant at $10^{-3}$ during the interleaved phase, and reverts to a one-cycle schedule with maximum $0.1$ during the retain-set fine-tuning phase.

\paragraph{Shakespeare.} We vary $q \in \{1,2,4\}$, $E_f \in \{5,7\}$, and $\lambda \in \{0.5, 1.0, 1.5\}$. The learning rate is held constant at $0.05$ during the interleaved phase and held constant at 0.1 during the retain finetuning phase.

%%\sss{Specify the exact learning-rate schedule for the retain-set fine-tuning phase on Shakespeare (the previous draft cut off mid-sentence).}

\subsection{Ascent on the forget set}\label{sec:appendix_ascent_on_forget}

\paragraph{Update rule.} The unlearning phase performs gradient ascent on the forget set,
\begin{equation}\label{eq:ascent_update}
x_{t+1} = x_t + \eta_t\, g_t^{f},
\end{equation}
where $g_t^{f}$ is the gradient of the forget loss at step $t$. After a fixed number of epochs (2 for both CIFAR-100 and Shakespeare), the algorithm switches to plain retain-set fine-tuning.

\paragraph{CIFAR-100.} The ascent learning rate is set to $10^{-3}$ during the unlearning phase, reverting to a one-cycle schedule with maximum $0.1$ during the retain-set fine-tuning phase.

\paragraph{Shakespeare.} The ascent learning rate is held constant at $0.05$ during the unlearning phase and held at 0.1 during the retain finetuning phase identical to the rate for interleaved ascent-descent.
\subsection{Pure fine-tuning on the retain set}\label{sec:appendix_pure_finetune}

The unlearning phase consists entirely of fine-tuning on the retain set, equivalent to a special case of model clipping with $C=C_2=\infty$ and $\sigma=0$. Learning-rate schedules match the corresponding retain-set fine-tuning phases described in \cref{sec:appendix_uncertified_methods}.

\section{Convex unlearning: auditor and hyper-parameter details}\label{sec:appendix_convex_unlearning}

\begin{definition}[$(\varepsilon,\delta)$-certified convex unlearning]\label{def:convex_unlearning}
	Given a dataset $\mcD$ of size $\npt$ and a forget set $\mcD_f\subseteq \mcD$ of size $\mpt$, let $\mcD_r:=\mcD\setminus\mcD_f$. An unlearning algorithm $\mcU$ is an $(\varepsilon,\delta)$-certified convex unlearning algorithm for $\mcA$ if $\mcU(\mcD_f,\mcA(\mcD),S(\mcD))$ and $\mcU(\emptyset,\mcA(\mcD_r),S(\mcD_r))$ are $(\varepsilon,\delta)$-indistinguishable.
\end{definition}

This appendix gives the auditor pipeline, the Perturbed Newton unlearning algorithm of \citet[Algorithm~1]{sekharimachineunlearningconvex}, the cubic and logistic loss constructions, and the hyper-parameter sweep used in \cref{sec:unlearning_algorithms}.

\subsection{Auditor pipeline}\label{sec:appendix_convex_auditor}

We adapt the pairwise auditor of \cref{par:pairwise_auditor} (\citealp{nasr2021adversaryinstantiationlowerbounds}-style) to the convex setting, with two changes. First, in line with \cref{def:convex_unlearning}, the auditor distinguishes the unlearned model produced from a non-empty forget set, $\mcU(\mcD_f,\mcA(\mcD),S(\mcD))$, against the model produced from an empty forget set, $\mcU(\emptyset,\mcA(\mcD_r),S(\mcD_r))$, rather than two unlearned models from different non-empty forget sets. Second, since the unlearned models live in a low-dimensional Euclidean parameter space, we instantiate the pairwise auditor's predictor directly on the model weights rather than on per-example logit scores: we train $\Gamma=50$ independent models under each setting, fit Gaussian distributions $\mathcal N(\mu_r,\Sigma_r)$ and $\mathcal N(\mu_f,\Sigma_f)$ to the resulting weight vectors, and at evaluation draw $2L$ test samples, each independently from one of the two fitted distributions. The auditor predicts the more likely generating distribution under these Gaussians, and the resulting empirical false-positive and false-negative rates are plugged into \cref{eq:pairwise_auditor_lb} to obtain $\varepsilon_{\mathrm{LB}}$. The null hypothesis $H_0: \varepsilon\le \varepsilon_{\mathrm{LB}}$ is then rejected with probability at most $\zeta$.

We additionally compute a second lower bound, $\varepsilon_{\mathrm{LB}}$, by instantiating \cref{metaalg:auditor_unlearning} with $m=2$ and $r=2$, so that a per-run overlap score of $2$ denotes a correct prediction and $0$ an incorrect one. The induced mechanism $\mcM$ is $(\varepsilon,\delta)$-DP by post-processing, and no division by $2$ is needed since no transitivity step is used.

\paragraph{Why MSE is excluded.} Mean squared error is somewhat degenerate in this context because its Hessian is constant. Consequently, in Algorithm~1 of \citet{sekharimachineunlearningconvex}, the Hessian estimate is exact, and by Lemma~3 therein the Newton-corrected iterate coincides with the model obtained by retraining on the retain set; the only remaining randomness comes from the final Gaussian perturbation step.

\subsection{Cubic loss in one dimension}\label{sec:appendix_convex_cubic_loss}

We instantiate the cubic loss in dimension $d=1$, so $w,z\in\mathbb R$. The pointwise loss is
\begin{equation}\label{eq:cubic_loss}
	\ell(w;z) = \tfrac{\lambda_0}{2}w^2 + \tfrac{M}{6}w^3 - zw,
\end{equation}
constrained to the box $|w|\le B$, on which the loss is $\mu$-strongly convex with $\mu=\lambda_0-MB$. We place the $\npt$ retain points at $z=0$ and the $\mpt$ forget points at $z=-R$, where $R>0$ is a data-radius parameter. Under this construction, the retain-only ERM has closed-form solution $\hat w_r=0$, while the forget points exert a constant gradient pull of magnitude $R$, pushing the unlearned weight $\hat w_f$ as far from $\hat w_r$ as the box constraint allows. In our experiments we set $M=0.005$, $B=2.5$, and $R=10.0$; the full hyper-parameter sweep over $(\lambda_0,M,B,R)$ is given below.

\subsection{Hyper-parameter selection}

Let $\ell(w;z)$ denote the pointwise loss at sample $z$, and let $\mcD = \mcD_r \cup \mcD_f$ be the full training set, where $\mcD_r$ and $\mcD_f$ are the retain and forget sets respectively.

Let
\[
\hat{w}_r \;=\; \argmin_w \sum_{z \in \mcD_r} \ell(w;z)
\]
denote the (noiseless) retain-only ERM, and let $\hat{w}_f$ denote the \emph{unlearned} weight, i.e.\ the weight obtained by first running ERM on the full dataset $\mcD$ to produce a trained model and \emph{then} applying the Hessian-based Newton unlearning steps of \citet[Eq.~8]{sekharimachineunlearningconvex} on the forget set $\mcD_f$ (with the Hessian evaluated on the retain set). In particular, $\hat{w}_f$ is \emph{not} the model produced by training alone; it is the post-unlearning weight before addition of noise. \citet[Lemma~3]{sekharimachineunlearningconvex} bounds
\[
\bigl\| \hat{w}_r - \hat{w}_f \bigr\|
\;\le\;
\frac{2 M m^2 L^2}{\mu^3 n^2},
\]
when $\ell(\cdot)$ is $\mu$-strongly convex, $L$-Lipschitz, and $M$-Hessian-Lipschitz in $w$. Our objective is to choose hyper-parameters that make the empirical $\|\hat{w}_r - \hat{w}_f\|$ as close to this theoretical bound as possible, i.e.\ that drive the ratio
\[
\rho \;:=\; \frac{\|\hat{w}_r - \hat{w}_f\|_{\text{empirical}}}{2 M m^2 L^2 / (\mu^3 n^2)}
\]
toward~$1$, since a value close to~$1$ indicates that Lemma~3 is empirically tight on the chosen problem instance. While the cubic loss typically yields values of $\rho$ around $0.2$, the logistic loss yields values of $\rho$ around $3\times 10^{-5}$, which possibly explains why the lower bound is so low.

\subsection{Placement of retain and forget points}

\subsubsection{Cubic loss.}
We instantiate the cubic loss in dimension $d=1$, so $w, z \in \mathbb{R}$ throughout this subsection. We construct $\mcD = \mcD_r \cup \mcD_f$ as follows. The retain set $\mcD_r$ consists of $n$ points placed at the origin,
\[
\mcD_r = \bigl\{(z_i, y_i)\bigr\}_{i=1}^{n}, \qquad
z_i = 0,\quad y_i = 0,
\]
and the forget set $\mcD_f$ consists of $m$ points placed at $-R$,
\[
\mcD_f = \bigl\{(z_i, y_i)\bigr\}_{i=1}^{m}, \qquad
z_i = -R,\quad y_i = 0,
\]
where $R > 0$ is the data-radius parameter. All labels are zero because the cubic loss
\[
f(w, z) \;=\; \tfrac{\lambda_0}{2}\, w^2 \;+\; \tfrac{M}{6}\, w^3 \;-\; z\, w
\]
does not depend on labels. Under this construction, the retain-only ERM admits the closed-form solution $\hat{w}_r = 0$, since the retain gradient $\lambda_0 w + \tfrac{M}{2} w^2$ vanishes at $w = 0$.

The retain points are placed at the origin so that $\hat{w}_r = 0$ provides a clean reference. The forget points are placed at $-R$ so as to maximise their gradient contribution to the full-data ERM: from
\[
\frac{\partial f(w, z)}{\partial w}\bigg|_{z = -R}
\;=\;
\lambda_0\, w + \tfrac{M}{2}\, w^2 + R,
\]
the forget set exerts a constant pull of magnitude $R$ on $w$ toward $+B$, pushing the unlearned weight $\hat{w}_f$ as far from $\hat{w}_r = 0$ as the box constraint $|w| \le B$ allows, and producing the tightest empirical Lemma~3 ratio.

For the cubic loss, we sweep $M \in \{0.001,\, 0.002,\, 0.005\}$, $B \in \{2.5,\, 3.0,\, 3.5\}$, $R \in \{4,\, 5,\, 6,\, 7,\, 8,\, 10,\, 12\}$, and $\lambda_0 \in \{0.10,\, 0.12,\, 0.15,\, 0.18,\, 0.20\}$, and obtain the best (largest) ratio of $\rho = 0.22$ at $M = 0.005$, $B = 2.5$, and $R = 10.0$. Because for the cubic loss (instantiated below in dimension $d=1$) we assume $|w|$ is bounded by $B$, which yields the strong-convexity parameter $\mu = \lambda_0 - M B$, we additionally require the optimum after loss minimisation to lie inside the interval $[-B, B]$; otherwise the first-order optimality condition $\nabla \ell = 0$ may fail to hold.

\subsubsection{Logistic loss (boundary construction).}{}
The retain set consists of $n$ points with i.i.d.\ Gaussian features and labels generated from a fixed direction $\theta^*$,
\[
\mcD_r = \bigl\{(x_i, y_i)\bigr\}_{i=1}^{n}, \qquad
x_i \overset{\text{i.i.d.}}{\sim} \mathcal{N}\!\left(0,\; \tfrac{1}{d} I_d\right), \quad
y_i \mid x_i \;\sim\; \mathrm{Bernoulli}\bigl(f_1(x_i^\top \theta^*)\bigr),
\]
where $f_1(z) = \eta \cdot \mathbf{1}[z \ge 0] + (1-\eta)\cdot \mathbf{1}[z < 0]$ is the label-flip function with flip probability $1-\eta$. The forget set places all $m$ points at the same fixed boundary location with a fixed label,
\[
\mcD_f = \bigl\{(x^*, y^*)\bigr\}_{i=1}^{m}, \qquad
x^* = -\,\mathrm{sign}(\theta^*) \cdot \rho_f, \quad y^* = 1,
\]
where $\rho_f > 0$ is the forget-point norm.

This construction is chosen for two reasons. First, concentrating all forget points at a single location $x^*$ makes every forget sample exert an identical, unidirectional gradient contribution, maximising the aggregate influence of $\mcD_f$ on the full-data optimum $\hat{w}_f$ relative to the retain-only optimum $\hat{w}_r$. Second, placing $x^*$ in the direction opposite to $\theta^*$ at distance $\rho_f$ from the origin ensures that the forget signal directly opposes the retain signal, creating the largest possible shift between $\hat{w}_f$ and $\hat{w}_r$, and hence the tightest empirical Lemma~3 ratio $\rho$. 
 However, the sweep for the logistic loss yielded a maximum ratio of only $3\times 10^{-5}$, which explains the poor $\varepsilon$ lower bound, at parameters $d=2$, $\eta=0.8$, $\zeta = 0.7$, and $\rho_f= 5.0$.

\begin{remark}
    One may wonder why we do not consider split size bigger than 2. While, we did try and consider splits bigger than 2 by increasing the dimension $d$ beyond 2 and ensure the forget points live along the each separate dimension. However, this does not yield higher lower bounds.  
\end{remark}

\section{Output perturbation: mechanism and point placement}\label{sec:appendix_output_perturbation}

\subsection{Mechanism}

Given a trained model $\hat{x}$, the output-perturbation mechanism produces the unlearned model
\[
x_u \;=\; \Pi_{C_0}(\hat{x}) \;+\; \xi_0, \qquad \xi_0 \sim \mathcal{N}(0,\, \sigma^2 \mathbb{I}),
\]
where $\Pi_{C_0}$ denotes Euclidean projection onto an $\ell_2$-ball of radius $C_0$ and $\sigma$ is chosen according to Theorem~3.1 of \citet{rewind2delete} so as to satisfy $(\varepsilon, \delta)$-certified unlearning. The auditing pipeline is identical to that of the cubic loss in \cref{sec:appendix_convex_auditor} (Gaussian fits on model parameters, Clopper--Pearson lower bound), with an MSE loss and no regulariser.

\subsection{Placement of retain and forget points}

The retain set consists of $n$ points with i.i.d.\ Gaussian features and continuous labels generated from a fixed direction $\theta^*$,
\[
\mcD_r = \bigl\{(x_i, y_i)\bigr\}_{i=1}^{n}, \qquad
x_i \overset{\text{i.i.d.}}{\sim} \mathcal{N}\!\left(0,\; \tfrac{1}{d} I_d\right), \quad
y_i \mid x_i = f_{\mathrm{lin}}(x_i^\top \theta^*) - \tfrac{1}{2} + \varepsilon_i,
\]
where
\[
f_{\mathrm{lin}}(z) \;=\; \mathrm{clip}\!\left(\tfrac{1}{2} + \eta z,\; \tfrac{1}{2} - r,\; \tfrac{1}{2} + r\right)
\]
is a clipped linear link with slope $\eta$ and half-range $r$, and $\varepsilon_i \sim \mathcal{N}(0, \sigma^2)$ is observation noise (clipped so that $y_i \in [-r, r]$) with $\sigma = 0.05$.

The forget set consists of $m$ i.i.d.\ points generated from the opposite direction
\[
\theta_f \;=\; -\theta^* \cdot \frac{\rho_f}{\|\theta^*\|},
\]
i.e.\
\[
\mcD_f = \bigl\{(x_j, y_j)\bigr\}_{j=1}^{m}, \qquad
x_j \overset{\text{i.i.d.}}{\sim} \mathcal{N}\!\left(0,\; \tfrac{1}{d} I_d\right), \quad
y_j \mid x_j = f_{\mathrm{lin}}(x_j^\top \theta_f) - \tfrac{1}{2} + \varepsilon_j,
\]
with $\|\theta_f\| = \rho_f \gg \|\theta^*\|$. Because $\theta_f$ points in the direction opposite to $\theta^*$, the forget labels are systematically negatively correlated with those in $\mcD_r$, so ERM on $\mcD_r \cup \mcD_f$ is pulled strongly away from $\theta^*$, maximising the gap between the unlearned model trained with $\mcD_f$ and the one trained with the empty forget set, and therefore tightening the resulting $\varepsilon_{\text{LB}}$.

The numbers of train samples $n$ and forget samples $m$ are $600$ and $400$ respectively, with $r = 10$, $\eta = 0.5$, $\zeta = 0.9$, and $\rho_f = 10$.

\begin{remark}
	One may wonder why output perturbation gives a much tighter lower bound than Perturbed Hessian unlearning (given in \cite{sekharimachineunlearningconvex}); this is because the norm bound ($\Delta$) used for the addition of Gaussian noise is closer to what is actually attained empirically. 
\end{remark}

\section{Computational resources}\label{sec:appendix_compute}

All experiments were performed on NVIDIA RTX 4000 GPUs. Each individual run (one train + unlearn cycle for a given configuration) took roughly $45$ minutes to $1$ hour. Since we audit several unlearning algorithms with multiple calibration and evaluation runs each, the full set of experiments reported in the paper took on the order of two weeks of wall-clock time.

\end{document}